\documentclass[a4paper,fleqn]{cas-dc}

\usepackage{hyperref}
\usepackage[numbers, square, comma, sort&compress]{natbib}
\usepackage[normalem]{ulem} 

\usepackage{empheq} 
\usepackage{xcolor}
\usepackage{ulem}
\usepackage{siunitx}
\usepackage{extdash}
\usepackage{comment}
\usepackage{soul}

\usepackage{bbm}
\usepackage{caption}
\usepackage{subcaption}

\usepackage{bm}
\usepackage{pifont}

\usepackage{amsthm}

\usepackage{numprint}
\npthousandsep{,}\npthousandthpartsep{}\npdecimalsign{.}

\newcommand{\BS}{\color{magenta}}

\newcommand{\cmark}{\ding{51}}%
\newcommand{\xmark}{\ding{55}}%

\def\tsc#1{\csdef{#1}{\textsc{\lowercase{#1}}\xspace}}
\tsc{WGM}
\tsc{QE}
\tsc{EP}
\tsc{PMS}
\tsc{BEC}
\tsc{DE}

\newcommand{\new}[1]{\textcolor{black}{#1}}
\newcommand{\del}[1]{\iffalse {#1} \fi}

\newtheorem{remark}{Remark}

\newtheorem{proposition}{Proposition}
\newtheorem{corollary}{Corollary}

\begin{document}
\let\WriteBookmarks\relax
\def\floatpagepagefraction{1}
\def\textpagefraction{.001}
\shorttitle{Towards Scalable Physically Consistent Neural Networks: an Application to Data-driven Multi-zone Thermal Building Models}
\shortauthors{Di Natale et~al.}


\title [mode = title]{Towards Scalable Physically Consistent Neural Networks: an Application to Data-driven Multi-zone Thermal Building Models}




\author[1,2]{Di Natale L.}[orcid=0000-0002-3295-412X]
\cormark[1]
\credit{Conceptualization, Methodology, Software, Validation, Formal analysis, Data Curation, Visualization, Writing - Original Draft}

\author[1]{Svetozarevic B.}
\credit{Conceptualization, Methodology, Writing - Review \& Editing,  Supervision }

\author[1]{Heer P.}
\credit{Writing - Review \& Editing, Resources, Funding acquisition}

\author[2]{Jones C.N.}
\credit{Conceptualization, Methodology, Writing - Review \& Editing, Supervision}

\address[1]{Urban Energy Systems Laboratory, Swiss Federal Laboratories for Materials Science and Technology (Empa), 8600 D\"{u}bendorf, Switzerland}
\address[2]{Laboratoire d'Automatique, EPFL, 1015 Lausanne, Switzerland}

\cortext[cor1]{Corresponding author:  \texttt{loris.dinatale@empa.ch} (L. Di Natale)}


\begin{abstract}
With more and more data being collected, 
data-driven modeling methods have been gaining in popularity in recent years. While physically sound, classical gray-box models are often cumbersome to identify and scale, and their accuracy might be hindered by their limited expressiveness. On the other hand, classical black-box methods, typically relying on Neural Networks (NNs) nowadays, often achieve impressive performance, even at scale, by deriving statistical patterns from data. However, they remain completely oblivious to the underlying physical laws, which may lead to potentially catastrophic failures if decisions for real-world physical systems are based on them. Physically Consistent Neural Networks (PCNNs) were recently developed to address these aforementioned issues, ensuring physical consistency while still leveraging NNs to attain state-of-the-art accuracy, \new{and applied to zone temperature modeling}. 

In this work, we scale PCNNs to model \new{the temperature dynamics of buildings with several connected thermal zones}\del{building temperature dynamics} 
and propose a thorough comparison with classical gray-box and black-box methods. More precisely, we design three distinct PCNN extensions \new{with different levels of information sharing between the modeled zones}, thereby exemplifying the modularity and flexibility of the architecture, and formally prove their physical consistency. In the presented case study, PCNNs are shown to achieve state-of-the-art accuracy, even outperforming classical NN-based models despite their constrained structure. Our investigations furthermore provide a clear illustration of NNs achieving seemingly good performance while remaining completely physics-agnostic, which can be misleading in practice. 
While this performance comes at the cost of computational complexity, PCNNs on the other hand show accuracy improvements of $17$ -- $35\%$ compared to all other physically consistent methods, paving the way for scalable physically consistent models with state-of-the-art performance.
\end{abstract}



\maketitle

    \section{Introduction}
    \label{sec:introduction}
    
Under the pressing issue of climate change, there is a worldwide effort to decrease our global energy consumption. Being responsible for a large part of the final energy consumption and greenhouse gas emissions~\cite{iea2020buildings}, buildings are a primary target in that trend, with space heating and cooling being the main identified culprits~\cite{doi/10.2833/525486}. Interestingly, we can intervene at any stage of the life of a building to decrease its energy intensity, either designing and constructing more efficient new structures~\cite{westermann2019surrogate}, retrofitting old edifices~\cite{deb2021review}, or designing smart controllers to minimize the energy consumption of the current building stock~\cite{svetozarevic2022data}. However, while decreasing the energy consumption of buildings is the main goal of advanced control algorithms, 
this cannot be done at the expense of the comfort of the inhabitants, who require the temperature to stay within a comfortable range~\cite{lei2022practical}. This calls for accurate building temperature models to close the sim2real gap of advanced control algorithms~\cite{hofer2021sim2real, kadian2020sim2real}. 
Indeed, Model Predictive Control (MPC) uses a model to predict the impact of possible power input sequences and choose the optimal one \cite{drgovna2020all}, for example, and Reinforcement Learning (RL) control policies have to be trained in simulation prior to their deployment \cite{lei2022practical, di2022near}. 

\subsection{Towards data-driven methods}

Since the evolution of the temperature in a thermal zone is governed by the laws of thermodynamics, the most natural way to model it is to write down the corresponding Ordinary Differential Equations (ODEs) and then use custom solvers or discretization schemes 
to propagate them through time, such as in~\cite{yang2021towards, dawood2022trade}. 
To alleviate the engineering burden of constructing the ODEs describing building temperature dynamics, 
allow more complex structures to be modeled, and accelerate the entire pipeline, 
custom modeling tools are often used in practice, such as EnergyPlus, Modelica, TRNSYS, or IDA ICE~\cite{crawley2001energyplus, wetter2006modelica, mazzeo2020energyplus}. Such detailed simulation tools however still require expert knowledge and access to many design parameters that are often not directly available \cite{harb2016development}, 
which makes them infamously hard to calibrate~\cite{zhang2019whole}. Moreover, solving the complex underlying ODEs to simulate each time step can entail a significant computational burden at runtime~\cite{ascione2017artificial}.

In recent years, owing to the growing amount of data collected in buildings, researchers started to leverage data-driven methods, bypassing the cumbersome procedures and expert knowledge required to set up classical physics-based models~\cite{bourdeau2019modeling}. This gave rise to so-called gray-box or black-box frameworks, both of which use historical data for calibration or training purposes, as pictured in Figure~\ref{fig:data-driven models}. 

    \begin{figure*}
    \begin{center}
    \includegraphics[width=\textwidth]{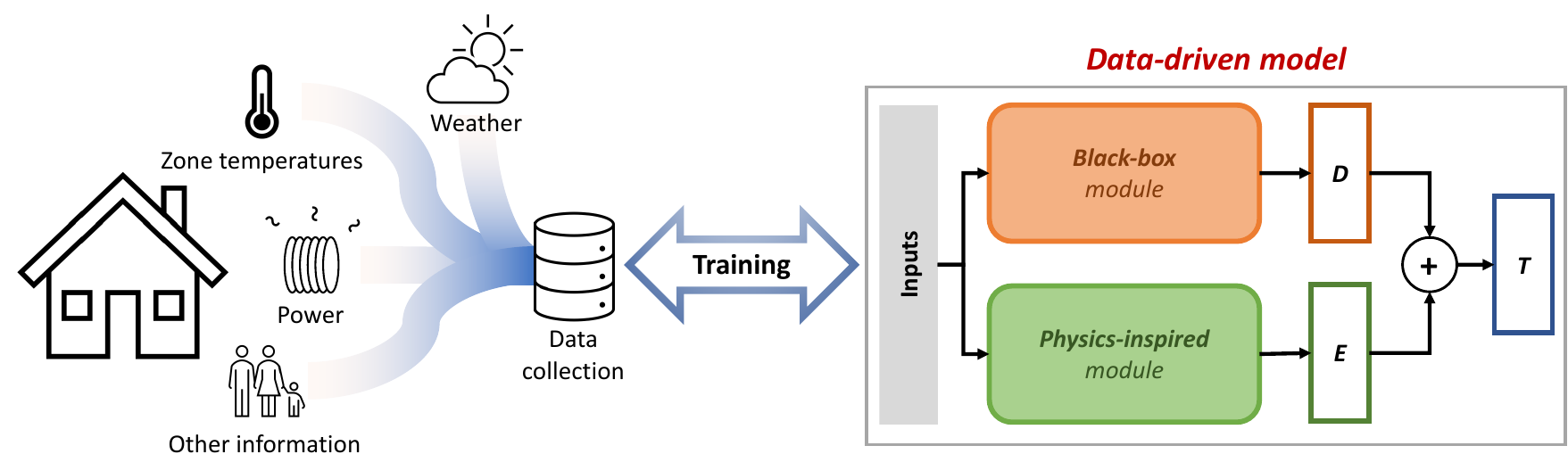}
    \caption{General pipeline of data-driven building thermal modeling frameworks.}
    \label{fig:data-driven models}
    \end{center}
    \end{figure*}

\subsection{Gray- and black-box models}

When a control-oriented thermal building model is designed, typically for MPC, data is in most cases used to identify the parameters of a simplified physics-based model~\cite{foucquier2013state}, usually a low-order Resistance-Capacitance (RC) model, such as in~\cite{maasoumy2014handling, maasoumy2014selecting, harb2016development, li2021grey, arroyo2020identification}. These models are particularly popular due to their ease of implementation, interpretability, close ties to the underlying physics, and because they often give rise to 
linear dynamics. The latter characteristic is indeed particularly desirable in MPC applications since an appropriate choice of objective function then renders the optimization problem to solve at each time step convex and hence tractable. Nonetheless, the parameter identification procedure of RC models is generally nontrivial and sensitive to the data quality~\cite{lin2012issues, shamsi2021feature}, which partially explains why low-order models often perform better than complex ones~\cite{shamsi2019generalization, berthou2014development}. Alternatively or additionally, when data is available, 
the residual errors of an often simplified model can be fit to improve its performance, such as in~\cite{sheng2020short}. 



While the aforementioned gray-box methods only require limited domain knowledge since simplified ODEs are used, which ODEs to choose is not always clear~\cite{li2014review} and wrong choices might hinder the performance of the identified model.
To avoid this pitfall, when enough data is available, fully data-driven black-box models might be used, such as AutoRegressive models with eXogenous inputs (ARX) or Neural Networks (NNs). Indeed, they do not rely on any prior knowledge of the system to model since they directly extract 
patterns from data to explain and predict the behavior of the system. Consequently, black-box methods are often easier and faster to deploy, more flexible, and thus often more scalable than their gray-box counterparts~\cite{royer2016towards, tien2022machine}. 
Furthermore, since they do not have to follow a predefined underlying architecture, black-box methods are generally more expressive, being capable of capturing unknown nonlinear dynamics, and hence usually perform better~\cite{deb2017review}. 

    \subsection{Neural Networks and their inconsistencies}

With the recent advances in Deep Learning (DL), NN-based solutions are gaining in popularity to represent unknown and potentially highly nonlinear dynamics, making them state-of-the-art solutions for time series modeling~\cite{abbasimehr2022improving}. 
Unsurprisingly, given the broad range of applications of NNs, researchers have already applied them to model building thermal dynamics, \new{for example}\del{e.g.} in ~\cite{zou2020towards, svetozarevic2022data, delcroix2021autoregressive}. When applying NNs to physical systems, one should however keep their well-known generalization issue in mind~\cite{szegedy2013intriguing}, which is partially caused by the \textit{underspecification} plaguing DL applications~\cite{d2020underspecification}. 
This can indeed be particularly problematic for NNs modeling physical processes since they are physics-agnostic and might hence find unrealistic solutions~\cite{deb2021review}. If the training data set does not span all the operating conditions of the system, there is thus a risk for NNs to fail to generalize in a meaningful manner to new conditions, an issue that is typically expected for thermal building models since the collected data sets are generally inherently incomplete~\cite{di2021physically}. 

Deep NNs are indeed able to learn \textit{shortcuts}~\cite{geirhos2020shortcut}, which means they might fit the training data well without fundamentally understanding the problem, hence failing to generalize. 
They for example attain superhuman performance on image recognition tasks, and yet fail when undistinguishable noise is added~\cite{szegedy2013intriguing} or the background changes~\cite{beery2018recognition, rosenfeld2018elephant}. They are also able to generate captions without ever looking at the corresponding images~\cite{heuer2016generating}, or detect pneumonia from X-ray scans 
only by looking at hospital-specific tokens 
and each hospital's pneumonia prevalence~\cite{zech2018variable}. 

While these are only a few examples, 
they clearly indicate how NNs can find ways to perform extremely well without fundamentally solving the task at hand. These flawed models are however unable to generalize and cannot be deployed in real-world applications since we have no means to know how they will react to new conditions. While tasks such as object recognition and captioning might be hard to characterize in general, the case of physical system modeling is different. Indeed, we often know the underlying physical laws and can hence impose constraints on NNs that help them understand the task at hand.


\subsection{Physics-inspired Neural Networks}

To incorporate some knowledge of the underlying physics in NN training procedures and promote desired system properties, counterbalancing the aforementioned brittleness of classical NNs, researchers recently proposed to design \textit{Physics-informed} or \textit{Phyiscs-inspired} NNs (PiNNs)~\cite{karniadakis2021physics}.
While many works modify the loss function of NNs to steer the learning towards physically meaningful solutions~\cite{daw2017physics, raissi2019physics}, these schemes cannot provide any guarantee about the final model respecting the desired constraints. More systematic approaches hence directly alter the networks' architecture to ensure the underlying physical laws are followed \textit{by design}, i.e.\new{,} at all times, such as in~\cite{greydanus2019hamiltonian, lutter2019deep, cranmer2020lagrangian, di2021physically}. Additionally, since the desired properties are hard-coded in such models, the loss function does not need to be altered, which avoids common pitfalls of classical PiNNs, such as the difficult trade-off between the accuracy and the physical consistency of the model, which can also increase the amount of data needed~\cite{hendriks2020linearly}. These specific NN architectures ensuring some physical properties are philosophically related to the celebrated convolutional \cite{lecun2010convolutional}, recurrent \cite{karpathy2015visualizing}, or graph \cite{scarselli2008graph} NN architectures, which were designed to capture spatial, timely, or neighboring relationships in the input data, respectively.

Despite the recent popularity of the field, to the best of the authors' knowledge, PiNNs were only applied to thermal building modeling in~\cite{gokhale2022physics, nagarathinam2022pacman, drgona2021physics, chen2023physics}. Gokhale et al. relied on the classical PiNN framework, augmenting the loss function of their NNs and creating latent states to include some physical intuition in otherwise standard networks~\cite{gokhale2022physics}, while Nagarathinam et al. designed a specific PiNN architecture for building control~\cite{nagarathinam2022pacman}. On the other hand, Drgo\v{n}a et al. used NNs to replace the matrices in linear models of building dynamics, which allowed them to enforce the stability and dissipativity of the learned system by constraining the eigenvalues of one of the NNs~\cite{drgona2021physics}. 
However, 
these works cannot provide guarantees about the physical consistency of their solutions in general beyond stability and dissipativity. 

On the contrary, the Physically Consistent Neural Networks (PCNNs) developed in previous work were theoretically proven to always yield physically consistent temperatures predictions, outperformed a classical gray-box model, and attained an accuracy on par 
with pure black-box models, but were limited to single-zone temperature modeling~\cite{di2021physically}. \new{PCNNs are composed of a physics-inspired and a black-box module running in parallel, the former ensuring compliance with the underlying physical laws and the latter capturing unmodeled and potentially highly nonlinear dynamics. In a concurrent line of work, so-called PC-NODEs leveraged Irreversible port-Hamiltonian systems to ensure satisfaction of the first and second law of thermodynamics \textit{by design}, relying on a similar training procedure as PCNNs~\cite{zakwan2022physically}. Despite being applied to the modeling of a three-zone building, however, this framework considered preprocessed solar gains as inputs instead of raw irradiation measurements, removing most of the nonlinear dynamics. In other words, PC-NODEs propose an alternative version of the physics-inspired module, but have yet to be applied to more complex case studies where significant nonlinearities are present.}

\subsection{Contribution}

In this work, we propose three different extensions of single-zone PCNNs~\cite{di2021physically} to model entire buildings, exemplifying how one can use the modularity of PCNNs to either expand the physics-inspired or black-box module. We then show that they all retain the desired physical consistency, both theoretically and with a numerical analysis, and investigate their performance on a three-zone building case study. Through extensive comparisons with several classical and state-of-the-art gray- and black-box methods, we demonstrate the ability of multi-zone PCNNs to leverage NNs to be more expressive than physically grounded gray-box methods and even outperform purely black-box NNs in terms of accuracy. Our investigations also clearly illustrate the phenomenon of shortcut learning on building temperature data, with NNs able to fit the data very accurately despite being completely oblivious to the underlying physics, which can be misleading in practice. 
Altogether, these experiments prove the effectiveness of the proposed PCNNs as thermal building models, alleviating the need for any engineering overhead, following the underlying physical laws, and reaching state-of-the-art accuracy.


The remainder of this paper is structured as follows. Section~\ref{sec:backround} first\del{ly} defines the notion of physical consistency and recalls the main principles behind PCNNs before the proposed multi-zone architectures are introduced and theoretically analyzed in Section~\ref{sec:methods}. Section~\ref{sec:case study} then details the case study, implementation considerations, and baseline models. 
Finally, the results are analyzed in Section~\ref{sec:results} and Section~\ref{sec:conclusion} concludes the paper.

    \section{Background}
    \label{sec:backround}
    

In this section, we recall a few prerequisites required to understand the methods proposed in this work, clarifying some definitions, formally introducing the notion of physical consistency, and recalling the design of single-zone PCNNs.

    \subsection{Definitions}

Two zones are said to be \textit{adjacent} if they share at least one common wall in a building, and the collection of zones adjacent to a given zone $z$ form its \textit{neighborhood} $\mathcal{N}(z)$. Note that we consider a zone to be included in its \new{own} neighborhood, i.e.\new{,} $z\in\mathcal{N}(z)$. Similarly, a zone is connected with the outside if it \new{comprises}\del{is composed of} at least one external wall. To generalize the notion of neighborhood, we define the \textit{$n$-hop neighborhood} $\mathcal{N}^n(z)$ as the set of zones that can be reached in $n$ steps from zone $z$, moving to an adjacent zone at each step. Note that, by definition, we have $\mathcal{N}^1(z)= \mathcal{N}(z)$, 
and $y\in\mathcal{N}^n(z)\iff z\in\mathcal{N}^n(y)$.

Throughout this work, we assume the building 
to be \textit{connected}, i.e.\new{,} there is no zone (or group of zones) isolated from the rest. This assumption is trivial in practice as one can easily train several separate models if this condition is not met.
    
    \subsection{Physical consistency}
    \label{sec:phys}
    
In this paper, we deem the temperature model of a building $\mathcal{B}$ with $m$ thermal zones to be \textit{physically consistent} if the following conditions are met for each zone $z\in\mathcal{B}$:
\begin{align}
    \frac{\partial T^z_{k+i}}{\partial u^z_{k+j}} &> 0 & \forall &0 \leq j<i, \label{equ:input consistency} \\ 
    \frac{\partial T^z_{k+i}}{\partial T^{out}_{k+j}} &> 0 & \forall & 0\leq j<i, \label{equ:outside consistency} \\
    \frac{\partial T^z_{k+i}}{\partial T^{y}_{k+j}} &> 0 & \forall & 0\leq j<i,\ \forall y\in\mathcal{N}^{i-j}(z), \label{equ:neighbors consistency} 
\end{align}
where $T^z_k$ is the temperature in zone $z$ at time step $k$, $u$ its heating/cooling power input, 
and $T^{out}$ represents the outside temperature. 

For example, \eqref{equ:input consistency} implies that 
applying more heating power $u^z_{k+j}$ 
at time step $k+j$ leads to higher temperatures $T^z_{k+i}$ for all subsequent time steps $i>j$. 
In other words, heating a zone has the expected and intuitive impact of increasing its temperature, following the laws of thermodynamics. \new{Note that cooling powers are defined to be negative in this work, hence inducing lower temperatures when more cooling is applied, as expected.} Similarly, \new{\eqref{equ:outside consistency} ensures that} higher ambient temperatures induce higher temperatures inside, and \new{\eqref{equ:neighbors consistency} guarantees that} higher temperatures in zone $y\in\mathcal{N}^n(z)$ lead to higher temperatures in zone $z$ after $n$ steps.
\del{Note that cooling powers are defined to be negative in this work, hence inducing lower 
temperatures when more cooling is applied, as expected.}

\begin{remark}[Generalization of the approach]\label{rem: generalization}
Note that the definition of physical consistency proposed in \eqref{equ:input consistency}-\eqref{equ:neighbors consistency} can easily be extended for applications where additional criteria need to be met by the learned model, to enforce physically consistent temperature predictions with respect to solar gains, for example. 
Interestingly, these conditions can also be seamlessly adapted to other fields beyond building modeling where simple physical rules can be encoded in a similar fashion. One can then construct a PCNN architecture following the principles presented in Sections~\ref{sec:PCNN} and~\ref{sec:methods} to ensure the learned model respect these desired criteria.
\end{remark}

    \subsection{Single-zone PCNNs}
    \label{sec:PCNN}

Conceptually, PCNNs are composed of a black-box and a physics-inspired module running in parallel to compute the next output at each step, as depicted \new{on the right of}\del{in} Figure~\ref{fig:data-driven models}. The former captures potentially complex nonlinearities while the latter ensures that predefined rules are respected, which typically represent physical laws and can be encoded by conditions similar to the ones proposed in \eqref{equ:input consistency}-\eqref{equ:neighbors consistency}. 
In the case of PCNNs modeling the evolution of the temperature in a single thermal zone $z$ while respecting the criteria in \eqref{equ:input consistency}-\eqref{equ:neighbors consistency}, they can be mathematically described as follows~\cite{di2021physically}: 
\begin{align}
    D_{k+1} &= D_k + f(x_k, D_k), \label{equ:PCNN-D}\\
    E_{k+1} &= E_k + a_h\max{\{u_k, 0\}} + a_c\min{\{u_k,0\}} \nonumber \\
            &\quad - b(T_k - T^{out}_k) - \sum_{z'\in\mathcal{N}(z)}{c_{z'}(T_k - T^{z'}_k)}, \label{equ:PCNN-E} \\
    T_{k+1} &= D_{k+1} + E_{k+1},  \label{equ:PCNN-T} \\
    D_k &= T(k), \nonumber \\
    E_k &= 0, \nonumber 
\end{align}
where $D\in\mathbb{R}$ represents the evolution of the black-box module based on a freely parametrized function $f:\mathbb{R}^{d+1}\to\mathbb{R}$, typically composed of NNs, and $E\in\mathbb{R}$ is the energy accumulator, i.e.\new{,} the physics-inspired module. The latter is influenced by the power inputs $u\in\mathbb{R}$ and heat transfers to the outside and adjacent zones. 
On the other hand, $x\in\mathbb{R}^d$ regroups all the \new{exogenous} inputs that are not \new{included}\del{treated} in $E$, \new{such as}\del{typically} time information and solar irradiation \del{measurements}. Finally, the constants $a_h$, $a_c$, $b$, and $\{c_{z'}\}_{z'\in\mathcal{N}(z)}$ capture the impact of heating, cooling, and heat losses to the outside and the neighboring zones on the modeled zone temperature, respectively. Applying \eqref{equ:PCNN-D}-\eqref{equ:PCNN-T} recursively over the prediction horizon, starting from the measured temperature $T(k)$ at time $k$, PCNNs can predict the evolution of the temperature while satisfying the criteria in \eqref{equ:input consistency}-\eqref{equ:neighbors consistency}. We refer the reader to the original paper for additional details~\cite{di2021physically}.

One important key to the effectiveness and generality of PCNNs comes from the fact that all the parameters $a_h$, $a_c$, $b$, $\{c_{z'}\}_{z'\in\mathcal{N}(z)}$, and $f$ are learned simultaneously using automatic BackPropagation Through Time (BPTT)~\new{\cite{di2021physically}}\del{\cite{werbos1990backpropagation} (Section~\ref{sec:implementations})}. As mentioned in Remark~\ref{rem: generalization}, the very generic structure of PCNNs can also be applied to model complex phenomena beyond thermal modeling, typically where only part of the physics is well understood. Indeed, it is always possible to adapt the structure of the physics-inspired module, which might also include nonlinearities, let the black-box module capture completely unknown dynamics in parallel, and seamlessly learn everything simultaneously in an end-to-end pipeline.

\begin{remark}[Consistency with respect to initial conditions] \label{rem:initial cond}
Since we want to model several zones in this paper, the condition \eqref{equ:neighbors consistency} is different from the one enforced in~\cite{di2021physically}. In particular, it includes the case $y=z$, $j=0$\new{:}\del{, i.e.} we want the derivative of any zone temperature with respect to its initial temperature to be positive to avoid spurious effects. This was not considered in the original paper on single-zone modeling and requires a slight modification of \eqref{equ:PCNN-D}, as discussed in the following section.
\end{remark}

    \section{Modeling the thermal behavior of buildings}
    \label{sec:methods}
    
This section proposes three different extensions of the single-zone PCNN architecture\new{, depicted in Figure~\ref{fig:PCNNs},} 
to model an entire building 
and then discusses their physical consistency.
    
    \subsection{Extensions to multi-zone PCNNs}
    \label{sec:multizone pcnns}

\begin{figure*}
     \centering
     \begin{subfigure}[b]{\columnwidth}
         \centering
         \includegraphics[width=\columnwidth]{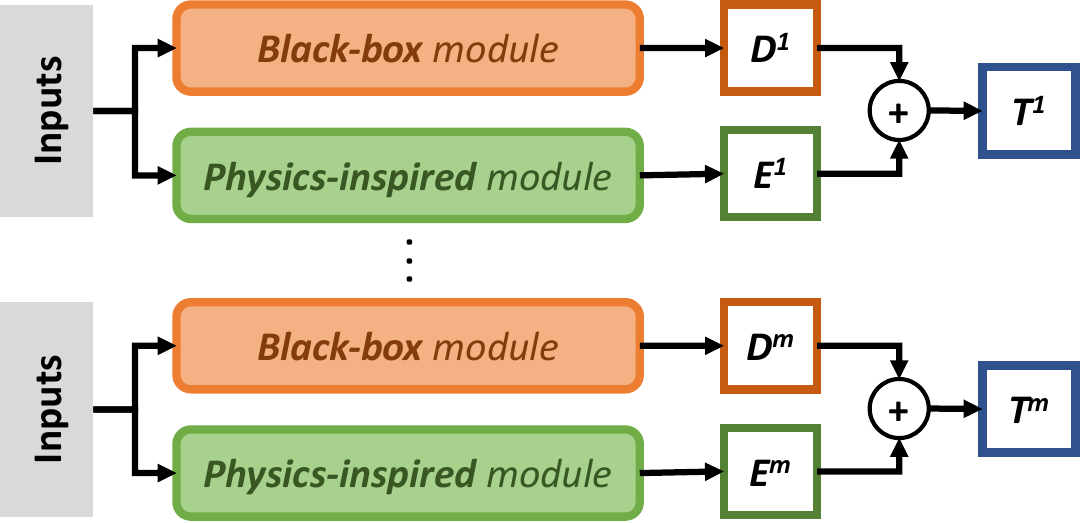}
         \caption{\new{\textbf{X-PCNN}: the temperature of each of the $m$ zones is predicted separately.}}
         \label{fig:X-PCNN}
     \end{subfigure}
     \hfill
     \begin{subfigure}[b]{\columnwidth}
         \centering
         \includegraphics[width=\columnwidth]{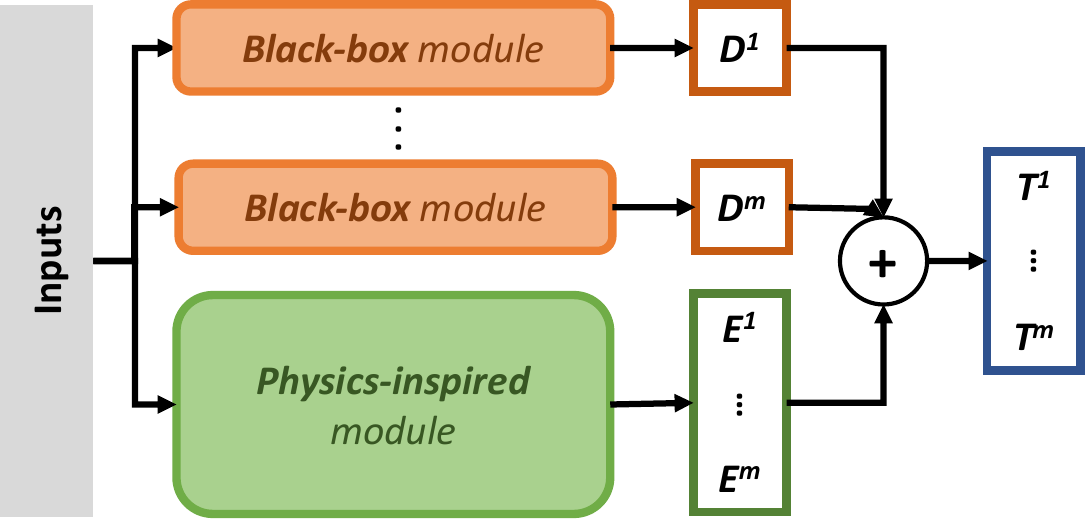}
         \caption{\new{\textbf{M-PCNN}: the physics-inspired module is shared but multiple black-box modules are learned, one for each zone.}}
         \label{fig:M-PCNN}
     \end{subfigure}
     \hfill
     \vfill\vspace{0.25cm}
     \begin{subfigure}[b]{\columnwidth}
         \centering
         \includegraphics[width=\columnwidth]{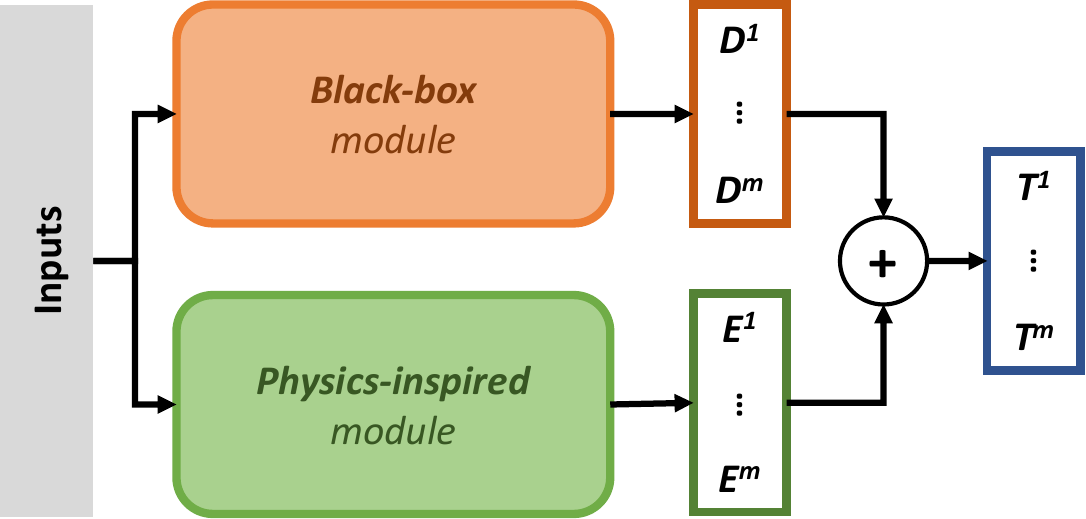}
         \caption{\new{\textbf{S-PCNN}: both the black-box and physics-inspired modules are shared by all the zones.}}
         \label{fig:S-PCNN}
     \end{subfigure}
        \caption{\new{The three PCNN architectures proposed in this work, with different levels of information sharing between the modeled thermal zones.}}
        \label{fig:PCNNs}
\end{figure*}

While the PCNN architecture described in Section~\ref{sec:PCNN} was shown to work well for single-zone modeling, this work proposes three possible extensions of this framework to simultaneously capture the evolution of the temperature in several interconnected zones exchanging energy, i.e.\new{,} in a whole building. The only additional information required is the \textit{topology} of the modeled building, i.e.\new{,} we assume to know which zones are adjacent and which have an external wall, and then learn its thermal behavior from data without engineering overhead. 

\begin{remark}[Topology]
If the topology is unknown, one can also assume each pair of zones to be adjacent and every zone to have an external wall and then learn to put non-existing connection parameters to zero from data. 
\end{remark}




    \subsubsection{\new{X-PCNNs: l}\del{L}earning several single-zone PCNNs}
    
The most natural and straightforward extension 
is to separately learn one PCNN for each of the zones to model, as depicted in Figure~\ref{fig:X-PCNN}. Since this method involves duplicating the original structure for each zone and fitting them independently, we will refer to the final model of the building as the \textit{X-PCNN} architecture. 
Mathematically, for a given zone $z$, we can write the corresponding equations as follows:
\begin{align}
    D^z_{k+1} &= D^z_k + f^z(x^z_k), \label{equ:X-PCNN-D}\\
    E^z_{k+1} &= E^z_k + a^z_h\max{\{u^z_k, 0\}} + a^z_c\min{\{u^z_k,0\}} \nonumber \\
            &\ \ \ - b(T^z_k - T^{out}_k) - \sum_{y\in\mathcal{N}(z)}{\tilde{c}^z_{y}(T^z_k - T^{y}_k)}, \label{equ:X-PCNN-E} \\
    T^z_{k+1} &= D^z_{k+1} + E^z_{k+1},  \label{equ:X-PCNN-T} \\
    D^z_k &= T^z(k), \nonumber \\
    E^z_k &= 0, \nonumber
\end{align}
where the superscript $z$ denotes zone-dependent information and all the variables have the same meaning as in Section~\ref{sec:PCNN}.

To retain the physical consistency of this model, however, one needs to ensure that 
$\tilde c^z_{y}=\tilde c^{y}_z$ for each pair of adjacent zones $y$ and $z$, so that the amount of energy flowing from $z$ to $y$ always equals the amount of energy received by $y$ from $z$, and vice versa. Since each zone is modeled and trained separately in this case, such a condition cannot be imposed during the learning phase, and we thus rely on a heuristic to correct the parameters and enforce this desired property \textit{a posteriori}. Once the models have been trained, for every pair of adjacent zones $z$ and $y$, we compute the average value identified by both PCNNs and define:
\begin{align}
    c^{zy} = c^{yz} = \frac{\tilde c^z_{y} + \tilde c^{y}_z}{2}. \label{equ:c correction}
\end{align}
For every zone $z$, we then replace $\tilde{c}^z_{y}$ with $c^{zy}$ in \eqref{equ:X-PCNN-E} for all $y\in\mathcal{N}(z)$. 


\begin{remark}[Independence of $f^z$ from $D^z$] \label{rem:D}
Note the difference between \eqref{equ:X-PCNN-D} and \eqref{equ:PCNN-D}, with $D^z$ not appearing in $f^z$ in \eqref{equ:X-PCNN-D}. Following Remark~\ref{rem:initial cond}, this ensures that condition \eqref{equ:neighbors consistency} is respected at all times, as analyzed in Section~\ref{sec:consistency} and Remark~\ref{rem:indep f}. 
\end{remark}

    \subsubsection{\new{M-PCNNs: s}\del{S}haring the physics-inspired module}

To avoid the hand-crafted correction \eqref{equ:c correction}, which might significantly impact the parameters learned by each PCNN, one can fuse all the physics-inspired modules together, again leveraging our prior knowledge of the underlying physical laws. This gives rise to the so-called \textit{M-PCNN} architecture, pictured in Figure~\ref{fig:M-PCNN}, where distinct black-box modules are assigned to each zone, but the physics-inspired module is shared and outputs a vector $\bm{E}\in\mathbb{R}^m$ containing the energy accumulated in each zone at each step:
\begin{align}
    \bm{E}_{k+1} &= \bm{E}_k + \bm{a}_h\odot\max{\{\bm{u}_k, \bm{0}\}} \nonumber \\
            &\quad+ \bm{a}_c\odot\min{\{\bm{u}_k,\bm{0}\}} \label{equ:M-PCNN-E} \\
            &\quad - \bm{b}\odot(\bm{T}_k - \new{\bm{T^{out}_k}}\del{T^{out}_k}) - \bm{\Delta T}_k, \nonumber \\
    \bm{E}_{k} &= \bm{0}, \nonumber
\end{align}
where the bold notations correspond to vectorized quantities in $\mathbb{R}^m$, one dimension for each zone, i.e.\new{,} $\bm{a_h} = [a_h^1, ..., a_h^m]^T$, and similarly for $\bm{a_c}$, $\bm{b}$, and $\bm{u_k}$, and $\odot$ stands for the element-wise product of two vectors. \new{Since there is a unique ambient temperature impacting all the zones, we furthermore define $\bm{T^{out}} = [T^{out}, ..., T^{out}]^T \in \mathbb{R}^m$.} Finally, $\bm{\Delta T}_k\in\mathbb{R}^m$ corresponds to energy transfer between each zone and its neighborhood:
\begin{align}
    \bm{\Delta T}_k^z &= \sum_{y\in\mathcal{N}(z)}{c^{zy}(T^z_k - T^{y}_k)}, \qquad\forall z\in\mathcal{B} \label{equ:deltaT}
\end{align}
where the subscript $z$ denotes the $z$-th entry of a vector. By definition, we know $c^{zy}=c^{yz}$ if $y$ and $z$ are adjacent since both represent the same heat transfer coefficient, 
which is easily enforced during training since all the zones are now modeled simultaneously, avoiding the \textit{a posteriori} correction \eqref{equ:c correction} required for X-PCNNs.

Each dimension of $\bm{T}\in\mathbb{R}^m$, i.e.\new{,} the temperature in each zone $z$, 
is then computed as the sum of the physics-inspired and black-box modules, as before:
\begin{align}
    \bm{T}_{k+1}^z &= D^z_{k+1} + \bm{E}_{k+1}^z, \label{equ:M-PCNN-T} \\
    D^z_{k+1} &= D^z_k + f^z(x^z_k), \label{equ:M-PCNN-D} \\
    D^z_{k} &= T^z(k). \nonumber
\end{align}

    \subsubsection{\new{S-PCNNs: s}\del{S}haring both modules}
    
To reduce the computational complexity of the model and introduce parameter sharing between the zones \new{---}\del{--} which are typically similar in the same building \new{---}\del{--}, we propose a third architecture, dubbed \textit{S-PCNN}, where both the black-box and physics-inspired modules are shared. In practice, this means that the black-box module, typically consisting of NNs, now has $m$ outputs corresponding to the main dynamics of each of the zones, as pictured in Figure~\ref{fig:S-PCNN}. Using the vectorized notations as before, i.e.\new{,} $\bm{D}\in\mathbb{R}^m$, we can write the equations of this architecture as follows:
\begin{align}
    \bm{D}_{k+1} &= \bm{D}_k + \bm{\tilde f}(\bm{\tilde x}_k), \label{equ:S-PCNN-D}\\
    \bm{E}_{k+1} &= \bm{E}_k + \bm{a}_h\odot\max{\{\bm{u}_k, \bm{0}\}} \nonumber \\
            &\quad+ \bm{a}_c\odot\min{\{\bm{u}_k,\bm{0}\}} \label{equ:S-PCNN-E} \\
            &\quad - \bm{b}\odot(\bm{T}_k - \new{\bm{T^{out}_k}}\del{T^{out}_k}) - \bm{\Delta T}_k, \nonumber \\
    \bm{T}_{k+1} &= \bm{D}_{k+1} + \bm{E}_{k+1},  \label{equ:S-PCNN-T} \\
    \bm{D}_k &= \bm{T}(k), \nonumber \\
    \bm{E}_k &= \bm{0}, \nonumber 
\end{align}
where the physics-inspired module is the same as for the \text{M-PCNN} but we now only have one shared nonlinear function $\bm{\tilde f}:\mathbb{R}^{d'}\to\mathbb{R}^{m}$ transforming the inputs $\bm{\tilde x}\in\mathbb{R}^{d'}$. Throughout this work, we only consider external inputs that are shared by all the zones, i.e. $d'=d$ and $\bm{\tilde x}:=x^z,\ \forall z\in\mathcal{B}$.

\begin{remark}[Zone-dependent inputs]
If some measurements differ zone by zone, 
one can either stack them in a vector $\bm{\tilde x} = [(x^1)^\top, ..., (x^m)^\top]^\top$ and use \eqref{equ:S-PCNN-D} as is or for example design a shared function 
$\bm{\tilde f}:\mathbb{R}^{d}\to\mathbb{R}$ and 
modify \eqref{equ:S-PCNN-D} to $\bm{D}_{k+1}^z = \bm{D}_k^z + \bm{\tilde f}(x^z_k)$, $\forall z\in\mathcal{B}$.
\end{remark}

    \subsection{Thermodynamical consistency}
    \label{sec:consistency}
    
Relying on the transformation \eqref{equ:c correction} ensuring that heat transfer coefficients between each adjacent zones are equal in the corresponding single-zone PCNNs, one can vectorize the X-PCNN physics-inspired module \eqref{equ:X-PCNN-E}, putting the parameters of each zone in vectors, to get:
\begin{align}
    \bm{E}_{k+1} &= \bm{E}_k + \bm{a}_h\odot\max{\{\bm{u}_k, \bm{0}\}} \nonumber \\
            &\quad+ \bm{a}_c\odot\min{\{\bm{u}_k,\bm{0}\}} \label{equ:common-E} \\
            &\quad - \bm{b}\odot(\bm{T}_k - \new{\bm{T^{out}_k}}\del{T^{out}_k}) - \bm{\Delta T}_k, \nonumber
\end{align}
using the definition of $\bm{\Delta T}$ from \eqref{equ:deltaT}. As can be seen directly, this expression is the same as the ones describing physics-inspired modules of the M-PCNN and S-PCNN architectures in \eqref{equ:M-PCNN-E} and \eqref{equ:S-PCNN-E}. This means all the proposed multi-zone PCNNs rely on the same physical model at inference time, with however possibly different parameter values learned during training. This is intuitively expected since they all model the same thermal effects and hence have to follow the same physical principles.

Similarly, we can rewrite the black-box modules of the X-PCNN and M-PCNN architectures in vectorized form as:
\begin{align}
    \bm{D}_{k+1} &= \bm{D}_k + \bm{\bar f}(\bm{\bar x}_k), \label{equ:common-D}
\end{align}
where  $\bm{\bar f}=[f^1(x^1), ..., f^m(x^m)]^T$ and $\bm{\bar x}=[(x^1)^\top, ..., (x^m)^\top]^\top$ groups the different inputs. 

For the three proposed architectures, putting \eqref{equ:common-E} and~\eqref{equ:common-D} together, we hence get:
\begin{align}
    \bm{T}_{k+1} &= \bm{D}_{k+1} + \bm{E}_{k+1} \nonumber\\
                &= \bm{T}_k + \bm{f}(\bm{x}_k) + \bm{a}_h\odot\max{\{\bm{u}_k, \bm{0}\}} \nonumber \\
                &\quad + \bm{a}_c\odot\min{\{\bm{u}_k,\bm{0}\}} \label{equ:common-T} \\
                &\quad - \bm{b}\odot(\bm{T}_k - \new{\bm{T^{out}_k}}\del{T^{out}_k}) - \bm{\Delta T}_k,  \nonumber \\
        \bm{D}_0&=\bm{T}(k), \nonumber \\
        \bm{E}_0&=\bm{0},  \nonumber
\end{align}
where $\bm{f}(\bm{x}_k)$ stands for $\bm{\tilde f}(\bm{\tilde x}_k)$ or $\bm{\bar f}(\bm{\bar x}_k)$ for S-PCNNs, respectively X- and M-PCNNs. The only structural difference between the three proposed models (once the heat transfer coefficients of the X-PCNN have been adjusted) hence comes from the form of $\bm{f}(\bm{x})$. Remarkably, however, this does not impact their physical consistency, as demonstrated in the following two propositions. 

\begin{proposition}[Heat propagation] \label{prop:heat propagation}
Independently of the structure of $\bm{f}$ and $\bm{x}$, any model of the form \eqref{equ:common-T} satisfies:
\begin{align}
    \frac{\partial\bm{T}_{k+i}^z}{\partial\bm{T}_{k+j}^y} &\geq 0 \qquad\forall z,y\in\mathcal{B},\ \forall 0\leq j<i, \label{equ:prop1}
\end{align}
with equality if and only if $y\notin\mathcal{N}^{(i-j)}(z)$, as long the following conditions hold:
\begin{align}
    b^z+\sum_{y\in\mathcal{N}(z)}c^{zy} &< 1, &\forall z&\in\mathcal{B}, \label{equ:prop1-cond1}\\
    c^{zy}&>0, &\forall z&\in\mathcal{B},\ \forall y\in\mathcal{N}(z). \label{equ:prop1-cond2}
\end{align} 
\end{proposition}
\begin{proof}
See Appendix~\ref{app:proof heat propagation}.
\end{proof}

In words, Proposition~\ref{prop:heat propagation} means that heat propagates from any zone $y$ to all the other zones $z$ as physically expected, i.e.\new{,} higher temperatures in a given zone $y$ will \del{eventually }lead to higher temperatures in all the other 
zones\del{, i.e.} after $(i-j)$ steps. 

Proposition~\ref{prop:heat propagation} can then be used to prove the following proposition stating that heating or cooling any zone ultimately increases, respectively decreases, the temperature in the whole building through heat transfers and that higher and lower ambient temperatures also impact the building as expected.


\begin{proposition}[Physical consistency with respect to inputs] \label{prop:inputs consistency}
Independently of the structure of $\bm{f}$ and $\bm{x}$, any model of the form \eqref{equ:common-T} satisfies:
\begin{align}
    \frac{\partial\bm{T}_{k+i}^z}{\partial\bm{u}_{k+j}^y} &\geq 0, & \forall &z,y\in\mathcal{B},\ \forall 0\leq j<i,\label{equ:prop2-1}
\end{align}
with equality if and only if $y\notin\mathcal{N}^{(i-j-1)}(z)$, and
\begin{align}
    \frac{\partial\bm{T}_{k+i}^z}{\partial T^{out}_{k+j}} &> 0, & \forall &z\in\mathcal{B},\ \forall 0\leq j<i, \label{equ:prop2-2}
\end{align}
as long as \eqref{equ:prop1-cond1}-\eqref{equ:prop1-cond2} hold and:
\begin{align}
    a^z_h, a^z_c, b^z &>0, & \forall &z\in\mathcal{B}. \label{equ:prop2-cond} 
\end{align} 
\end{proposition}
\begin{proof}
See Appendix~\ref{app:proof inputs consistency}.
\end{proof}

\begin{corollary}[Physical consistency of PCNNs] \label{cor:consistency}
Independently of the structure of $\bm{f}$ and $\bm{x}$, any model of the form \eqref{equ:common-T}
respects the physical consistency criteria~\eqref{equ:input consistency}-\eqref{equ:neighbors consistency} if:
\begin{align}
    b^z+\sum_{y\in\mathcal{N}(z)}c^{zy} &< 1, &\forall z&\in\mathcal{B}, \label{equ:cor-cond1}\\
    a^z_h, a^z_c, b^z, c^{zy}&>0, &\forall z&\in\mathcal{B},\ \forall y\in\mathcal{N}(z). \label{equ:cor-cond2}
\end{align} 
\end{corollary}
\begin{proof}
Assuming that \eqref{equ:cor-cond1} and \eqref{equ:cor-cond2} hold, we can apply Propositions~\ref{prop:heat propagation} and \ref{prop:inputs consistency}. Setting $z=y$ in \eqref{equ:prop2-1} and recalling that any zone is in its own neighborhood \new{---}\del{--} which implies \new{strict} positiveness of \eqref{equ:prop2-1} \new{---}\del{--}, PCNNs satisfy \eqref{equ:input consistency}. 
The satisfaction of \eqref{equ:outside consistency} directly follows from the second part of Proposition~\ref{prop:inputs consistency}. Finally, according to Proposition~\ref{prop:heat propagation}, \eqref{equ:prop1} is strictly positive if and only if $y\in\mathcal{N}^{(i-j)}(z)$, satisfying \eqref{equ:neighbors consistency}.
\end{proof}

This corollary thus proves that each of the proposed PCNN architectures remains physically consistent 
as long as all the parameters $\bm{a_h}$, $\bm{a_c}$, $\bm{b}$, and $\bm{c}$ are small positive constants. Note that this makes intuitive sense since all these parameters correspond to inverses of resistances and capacitances, hence small positive numbers, in real buildings. Interestingly, these conditions can easily be enforced during the training procedure without modifying the classical \new{backpropagation through time}\del{BPTT} algorithm, hence allowing us to rely on well-developed tools to train our models, as detailed in Section~\ref{sec:implementations}.


Remarkably, the strength of our approach lies in the fact that all the models will remain consistent whatever the structure of $\bm{f}$ is, being shared or not\new{,}\footnote{This is the main difference between the M-PCNN and S-PCNN architectures in our case, for example.}\del{,} composed of NNs or other nonlinearities. 
This gives the user complete freedom in the design of the black-box module without jeopardizing the physical consistency of the model. Similarly, all the parameters of the physics-inspired module might for example be time-varying or computed as nonlinear functions of external inputs without impacting the consistency of the model as long as they stay small and positive at all times.

\begin{remark}[Inputs of $\bm{f}$] \label{rem:indep f}
While the structure of $\bm{f}$ 
does not impact the validity of Propositions \ref{prop:heat propagation} and \ref{prop:inputs consistency}, its inputs do. In particular, $\bm{f}$ has to be independent of $\{T, u, T^{out}\}$ for the first step of the proofs of both propositions to hold in general (Appendix~\ref{app:proofs}). If $\bm{f}=\bm{f}(\bm{x},\bm{D})$ for example, it would modify the case $z=y$ in equation~\eqref{equ:base} in Appendix~\ref{app:proof heat propagation}, and the satisfaction of \eqref{equ:prop1-cond1} would then not be sufficient to guarantee the required nonnegativity of the partial derivatives in \eqref{equ:prop1}. 
\end{remark}

\begin{remark}[A control perspective]
    The multi-zone PCNNs \eqref{equ:common-T} are power input-affine. This makes such models interesting in control applications, typically for MPC schemes aimed at decreasing the energy consumption of buildings.
\end{remark}

    \section{Case study}
    \label{sec:case study}
    
To assess the quality of the multi-zone PCNN architectures detailed in Section~\ref{sec:methods}, we carry out an extensive performance analysis on a case study, where the objective is to predict the temperature dynamics over three-day-long horizons with \SI{15}{\minute} time steps. This section presents the building where the data was collected, implementation details, and then introduces the other gray- and black-box methods used as benchmarks.

    \subsection{Dataset}
    
    \begin{figure}
    \begin{center}
    \includegraphics[width=\columnwidth]{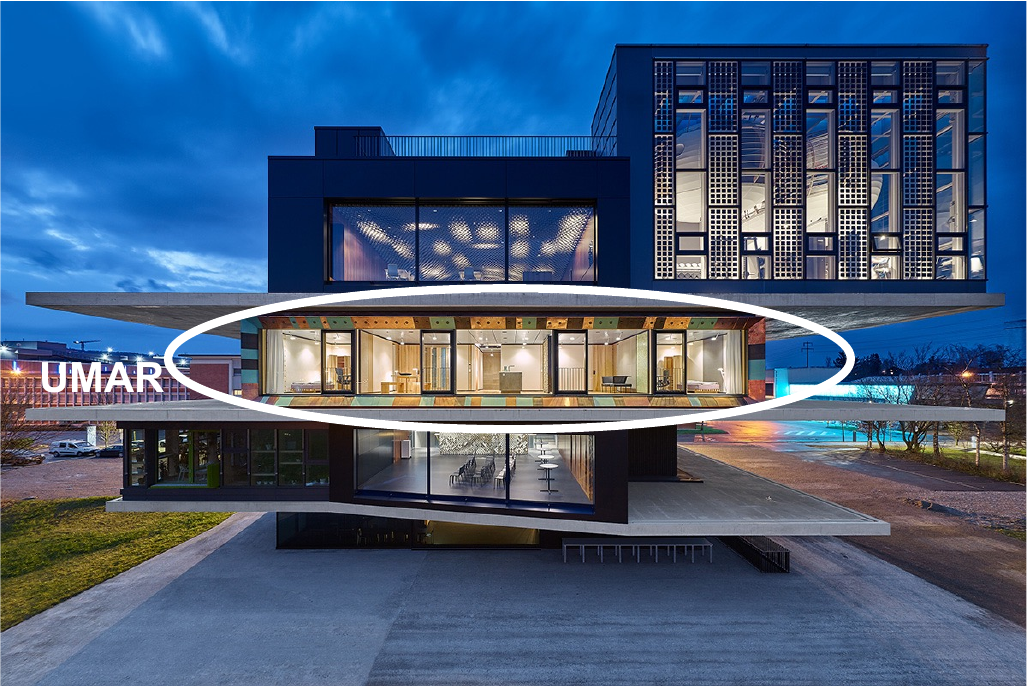}
    \caption{NEST, Duebendorf, Switzerland, with UMAR highlighted in white {\copyright} Zooey Braun, Stuttgart.}
    \label{fig:Nest}
    \end{center}
    \end{figure}
    
The data used in this work was collected in NEST, a vertically integrated district composed of several units~\cite{nest}, located in Duebendorf, Switzerland, and pictured in Figure~\ref{fig:Nest}. 
In this paper, we focus on the Urban Mining and Recycling (UMAR) unit, circled in white in Figure~\ref{fig:Nest}, an apartment with two bedrooms and a living room in between them. We are thus modeling three thermal zones arranged in a line in this work, i.e.\new{,} Zone $1$ is connected to Zone $2$, and Zone $2$ is also connected to Zone $3$, and each of them has at least one external wall. In heating mode, all the zones are heated by letting hot water flow through ceiling panels. During the cooling season, on the other hand, cold water can flow through the panels to cool down the zones.

We rely on three years of data collected between May 2019 and May 2022 and preprocessed as explained in~\cite[App. C]{di2021physically}. This involved downsampling the data to \SI{15}{\minute} intervals, smoothing the time series, and disaggregating 
the thermal power consumption of UMAR into the consumption of each zone. Besides these computed thermal power inputs, the dataset also contains measurements of the temperature in each zone and outside, the horizontal solar irradiation on-site, and the status of the system, i.e.\new{,} if it is in heating or cooling mode. 
To facilitate the learning process of the NNs in the black-box modules, we completed the data with additional time information, i.e.\new{,} the day of the week and the sine and cosine transformations of the time of the day and the month~\cite[App. D]{di2021physically}. Finally, we split the data in a training and a validation set, respectively denoted $\mathcal{D}_{t}$ and $\mathcal{D}_{v}$, containing possibly overlapping time series of up to three days of data, as detailed in Appendix~\ref{app:processing}.

    \subsection{Implementation details}
    \label{sec:implementations}
    
To ensure the physical consistency of the proposed multi-zone PCNNs, i.e.\new{,} to fulfill the conditions \eqref{equ:cor-cond1} and~\eqref{equ:cor-cond2}, we parametrize the log-value of each parameter, i.e.\new{,} we learn $\tilde a^z_h, \tilde a^z_c, \tilde b^z, \tilde c^{zy}$, $\forall \new{z}\del{y}\in\mathcal{B}$, $\forall y\in\mathcal{N}(z)$ and define:
\begin{align}
    s &= s_0\exp{(\tilde s)}, &\forall s=\{a^z_h, a^z_c, b^z, c^{zy}\}
\end{align}
where $s_0$ is the initial value of the parameter, defined using the same rules of thumb as in~\cite{di2021physically}. Starting from $\tilde{s}=0$, PCNNs hence learn to scale the initial value $s_0$ instead of modifying it directly, which is more numerically stable and ensures that $s$ stays small enough, while the exponential function keeps all the parameters positive at all times. \new{Note that these parameters are learned simultaneously to the parameters in the black-box module: when backpropagation is used to update the parameters of the NNs, we also leverage the propagated gradients to update the parameters of the physics-inspired module.}

\begin{remark}[Upperbound on $s$]
While $\tilde{b}^z$ or $\tilde{c}^{zy}$ can in principle grow uncontrollably and lead to a violation of the necessary condition \eqref{equ:cor-cond1}, 
this was not an issue in our experiments. 
Nonetheless, one can always introduce bounds on the learned values $\tilde{s}$ if required, typically leveraging activation functions like the sigmoid or hyperbolic tangent to control the range of learned values.
\end{remark}

In this paper, each function $\bm{f}$ has the same encoder-LSTM-decoder architecture, which is repeated when several modules are required for X-PCNNs and M-PCNNs. 
Both the encoder and decoder are feedforward NNs with $32$ hidden units and the LSTM comprises two layers with dimension $64$ and is followed by a normalization layer. \new{This architecture was selected over larger ones since we did not observe any significant decrease in performance.}\del{These numbers of neurons were chosen over larger ones since we did not observe any significant increase in performance with larger architectures.} The input of each black-box module, $\bm{x}\in\mathbb{R}^6$, gathers the solar irradiation on a horizontal surface and the time information. NN\del{-based model}s share a common learning rate of $5\new{\texttt{e}}\del{e}{-4}$\new{, manually selected small enough to ensure stable convergence,} and \new{a} batch size of \new{\numprint{4096}}\del{$4096$, and }\new{, to maximize the utility of the Graphical Processing Units (GPUs). }\new{Every model} minimize\new{s} the Mean Square Error (MSE) over 
a given batch of data $B$:
\begin{align}
    \new{\mathcal{L}_\textit{data}}\del{\mathcal{L}_{data}} &= \frac{1}{|B|}\sum_{s\in B}\left[\frac{1}{l_s}\sum_{k=0}^{l_s-1}\left[\frac{1}{m}\sum_{z=1}^{m}\xi^{z,s}_k\right]\right], \label{equ:loss_data}\\
    \xi^{z,s}_k &= \left(\bm{T}^{z,s}_{k+1} - T^{z,s}(k+1)\right)^2, \nonumber
\end{align}
where the subscript $s$ corresponds to which time series of length $l_s$ the predictions and measured temperatures are taken from. 
To then validate the performance of a model, we rely on the Mean Absolute Error (MAE), where 
\begin{align*}
\xi^{z,s}_k &= |\bm{T}^{z,s}_{k+1} - T^{z,s}(k+1)|,
\end{align*}
\new{in~\eqref{equ:loss_data}} and the Mean Absolute Percentage Error (MAPE), with:
\begin{align*}
\xi^{z,s}_k &= \frac{\new{|}\bm{T}^{z,s}_{k+1} - T^{z,s}(k+1)\new{|}}{T^{z,s}(k+1)}.
\end{align*}
The PCNNs are implemented in PyTorch~\cite{NEURIPS2019_9015} 
and were trained on NVIDIA P100 \new{GPUs}\del{Graphics Processing Units (GPUs)}. The code and data 
\new{are}\del{will be made} available on \url{https://gitlab.nccr-automation.ch/loris.dinatale/multi-zone-pcnns}.

    \subsection{Benchmark models}
    \label{sec:benchmark}
    
To analyze the performance of the proposed PCNN architectures, we perform an extensive ablation study and compare them to state-of-the-art gray- and black-box methods. An overview of all the models used in this work, and whether they are physically consistent, can be found in Table~\ref{tab:accuracy}\del{Authors: Old Table 1 has been deleted}. Note that the linear and LSTM models correspond to learning only the physics-inspired module of S-PCNNs, respectively the black-box one. Furthermore, \textit{Res-cons} is equivalent to fitting both modules of S-PCNNs sequentially, showcasing the importance of learning all the parameters of PCNNs simultaneously to attain state-of-the-art accuracy.

    \subsubsection{Linear gray-box model}
    \label{sec:linear}
    
First\del{ly}, it makes intuitive sense to investigate the accuracy of the physics-inspired module of PCNNs on its own, leading to the following linear gray-box model, hereafter referred to as the \textit{Linear} model:
\begin{align}
    \bm{T}_{k+1} &= \bm{T}_k + \bm{a}_h\odot\max{\{\bm{u}_k, \bm{0}\}} \nonumber \\
                &\quad + \bm{a}_c\odot\min{\{\bm{u}_k,\bm{0}\}} \label{equ:linear} \\
                &\quad - \bm{b}\odot(\bm{T}_k - T^{out}_k) - \bm{\Delta T}_k  + \bm{e}\odot\bm{Q}^{win}_k, \nonumber
\end{align}
where $\bm{Q}^{win}_k$ gathers the solar irradiation on the windows of each zone in a vector, engineered from the measured irradiation on a horizontal surface. 
Since there is no black-box module taking care of the impact of the sun on building temperatures in this model, we indeed need to include it manually. This can be done efficiently for UMAR but does not generalize to arbitrary buildings, limiting the applications of such linear models, as detailed in Appendix~\ref{app:solar preprocessing}. As for the other heat gains, $\bm{e}$ gathers the trainable scaling parameters reflecting the impact of solar gains on each zone temperature in a vector. 

Since the classical least squares parameter identification gave rise to physically inconsistent parameters, we chose to identify $\bm{a}_h,\bm{a}_c,\bm{b},\bm{c},\bm{e}$ for each zone using Bayesian Optimization (BO), as detailed in Appendix~\ref{app:linear}. As for X-PCNNs, the heat transfer coefficients between two adjacent thermal zones were then averaged based on \eqref{equ:c correction}. 

    \subsubsection{Residual models}
    \label{sec:residual}

A natural extension of the aforementioned linear model is to consider \textit{residual models}, where the idea is to fit the errors of the linear model predictions with a black-box module to improve its performance. Assuming the linear model in \eqref{equ:linear} to provide predictions $\bm{\hat T}_{k+1}$, 
a residual model fits a function $\new{f_\textit{Res}}\del{f_{Res}}:\mathbb{R}^{d'+2m+1}\to\mathbb{R}^m$, typically modeled with NNs, to the residual errors, i.e.\new{,} it minimizes:
\begin{align}
    \new{\mathcal{L}_\textit{Res}}\del{\mathcal{L}_{Res}} &= \frac{1}{|B|}\sum_{s\in B}\left[\frac{1}{l_s}\sum_{k=0}^{l_s-1}\left[\frac{1}{m}\sum_{z=1}^{m}(\epsilon^{z,s}_k)^2\right]\right], \label{equ:res} \\
    \epsilon^{z,s}_k &= \new{f_\textit{Res}^z}\del{f_{Res}^z}(\bm{T}_k^{s}, \bm{x}_k^{s}, \bm{u}_k^{s}, T^{out,s}_k) \nonumber \\
    &\qquad- \left(\bm{T}^{z,s}(k+1) - \bm{\hat T}_{k+1}^{z,s}\right), \nonumber
\end{align}
and then predicts temperatures as follows:
\begin{align}
    \bm{T}_{k+1} &= \bm{\hat T}_{k+1} + f_{Res}(\bm{T}_k, \bm{x}_k, \bm{u}_k, T^{out}_k). \label{equ:res 1}
\end{align}
This model is dubbed \textit{Res} in the rest of this paper, and $f_{Res}$ has the same encoder-LSTM-decoder structure as the proposed PCNNs for fair comparisons. 

Remarkably, such residual models cannot be ensured to respect the underlying physical laws in general since $\new{f_\textit{Res}}\del{f_{Res}}$ is not independent of zone temperatures, power inputs, and ambient temperatures. Since they are composed of a physics-inspired base model and a black-box module running in parallel, as the proposed PCNN architectures, we can indeed use similar arguments to prove their physical consistency (see Remark~\ref{rem:indep f}). 
Consequently, we also investigate the performance of a physically consistent residual model in this work, dubbed \textit{Res-cons}, where the black-box function learning the residuals only depends on $\bm{x}$, as PCNNs. 
This model hence fits a function $\new{f_\textit{Res-cons}}\del{f_{Res-cons}}:\mathbb{R}^{d'}\to\mathbb{R}^m$ to the residuals, trained similarly to its physically inconsistent counterpart, with the following temperature predictions: 
\begin{align}
    \bm{T}_{k+1} &= \bm{\hat T}_{k+1} + \new{f_\textit{Res-cons}}\del{f_{Res-cons}}(\bm{x}_k). \label{equ:res 2}
\end{align}

Note that residual models first fit the base model to the data and then use black-box methods to fit the residual errors while PCNNs learn both modules together. This also implies that the physics-inspired module reflects the main dynamics of residual models while it only ensures the physical consistency of PCNN architectures, letting more expressive functions like NNs capture the main system dynamics. 

    \subsubsection{Autoregressive model with exogenous inputs}
    
As a first black-box method, we analyze the performance of an ARX model, 
where autoregressive lags of the states and inputs are used to predict the next state:
\begin{align}
    \bm{T}_{k+1} &= \bm{\alpha}_0\bm{T}_{k} + \bm{\alpha}_1\bm{T}_{k-1} + \hdots + \bm{\alpha}_\delta\bm{T}_{k-\delta} \nonumber \\
    &\quad + \bm{\beta}_0 \bm{\hat{x}}_k + \bm{\beta}_1 \bm{\hat{x}}_{k-1} + \hdots + \bm{\beta}_\delta\bm{\hat{x}}_{k-\delta}, \\
    \bm{\hat{x}}_k &= [\bm{u}_k, T^{out}_k, Q^{sun}_k]^T, \nonumber 
\end{align}
where $Q^{sun}_k\in\mathbb{R}$ is the solar irradiation measurement on a horizontal surface, and the parameters $\bm{\alpha}_0, \hdots, \bm{\alpha}_\delta \in\mathbb{R}^{m\text{x}m}$, $\bm{\beta}_0, \hdots, \bm{\beta}_\delta \in\mathbb{R}^{m\text{x}(m+2)}$ are identified through least square regression using the \texttt{scikit-learn} library~\cite{scikit-learn}. \new{For a fair comparison}\del{In this work}, we set $\delta = 11$, i.e.\new{,} we use information from the last \SI{3}{\hour} to define the next temperatures, similarly to the warm start period of the proposed PCNN architectures.

For comparison purposes, we also implemented an advanced ARX model relying on the \texttt{statsmodels} package~\cite{seabold2010statsmodels}, which includes Kalman smoothing and filtering operations out-of-the-box and sets $\bm{\beta}_1, \hdots, \bm{\beta}_\delta=0$ so that only current information on external inputs is used. Since the identification procedure was harder in that case, we identified each zone $z$ separately, with:
\begin{align*}
    \bm{\hat{x}}_k^z &= [\bm{u}_k^z, T^{out}_k, Q^{sun}_k, \bm{T}^{y_1}_k, \bm{T}^{y_2}_k, \hdots, \bm{T}^{y_{|\mathcal{N}(z)|}}_k]^T,
\end{align*}
where $y_1, \hdots, y_{|\mathcal{N}(z)|} \in \mathcal{N}(z)$ are the zones adjacent to $z$, and the final model is dubbed \textit{ARX-KF}. Note that ARX models cannot be enforced to be physically consistent in general. 

\begin{remark}[Kalman filtering and smoothing]
Note that these operations could be included for any other model as well, potentially impacting their performance. In this work, we focus on methods working on unfiltered data, typically involving NNs, but we also provide this \textit{ARX-KF} as an example of what can be achieved with existing toolboxes on a laptop compared to NN-based methods \new{that might require}\del{requiring} access to GPUs for training.
\end{remark}

    \subsubsection{Neural network models}

As another natural ablation of PCNNs, 
we also investigate the quality of the black-box module alone. Instead of treating the power inputs and temperatures in a separate module, all the inputs are fed in the black-box function $\new{f_\textit{LSTM}}\del{f_{LSTM}}:\mathbb{R}^{d'+2m+1}\to\mathbb{R}^m$, leading to the \textit{LSTM} model: 
\begin{align}
    \bm{T}_{k+1} &= \bm{T}_{k} + \new{f_\textit{LSTM}}\del{f_{LSTM}}(\bm{T}_{k}, \bm{u}_{k}, \bm{x}_k, T^{out}_k).
\end{align}
As expected, such classical NN-based methods are naturally physically inconsistent and might fail to capture the underlying physical laws even if they fit the data well (Section~\ref{sec:introduction}).

Finally, 
we also compare PCNNs to a standard physics-informed NN, hereafter the \textit{PiNN} model, again relying on the same architecture as the black-box modules of PCNNs and the LSTM model. However, as is classically done, its loss function is modified to steer the learning toward physically meaningful solutions, with:
\begin{align}
    \new{\mathcal{L}_\textit{PiNN}}\del{\mathcal{L}_{PiNN}} &= \new{\mathcal{L}_\textit{data}}\del{\mathcal{L}_{data}} + \lambda \new{\mathcal{L}_\textit{phys}}\del{\mathcal{L}_{phys}},
\end{align}
where $\lambda$ is a tuning hyperparameter. Since the purpose of this additional loss term $\new{\mathcal{L}_\textit{phys}}\del{\mathcal{L}_{phys}}$ is to capture physical inconsistencies and penalize them, we naturally design it to bias the model towards solutions satisfying the desired properties \eqref{equ:input consistency} and~\eqref{equ:outside consistency}. Consequently, we penalize negative gradients of the final predicted temperatures, i.e.\new{,} at time $l_s$, with respect to control inputs and ambient temperatures observed along the horizon:
\begin{align}
    \new{\mathcal{L}_\textit{phys}}\del{\mathcal{L}_{phys}}&= \frac{1}{|B|}\sum_{s\in B}\left[\frac{1}{l_s}\sum_{k=0}^{l_s-1}\left[\frac{1}{m}\sum_{z=1}^{m}g^{z,s}_k\right]\right], \label{equ:phys loss}\\
     g^{z,s}_k &= \sum_{y=1}^m\left[r\left(-\frac{\partial \bm{T}^{z,s}_{l_s}}{\partial \bm{u}^{y,s}_{k}}\right)\right] + r\left(-\frac{\partial \bm{T}^{z,s}_{l_s}}{\partial T^{out,s}_{k}}\right), \label{equ:gradients}
\end{align}
where $r(x)=\max\{x,0\}$, also known as the Rectified Linear Unit (ReLU) function. Since we are interested in physically consistent models in this work, we empirically fixed $\lambda=100$, which ensures the loss term is dominated by the physical inconsistencies, thereby steering the PiNN towards interesting solutions.

\new{Note that, in building temperature modeling, one can also augment the outputs of NNs to predict not only zone temperatures but also the temperatures of their respective thermal mass, for example, and then penalize deviations of the latter from the predictions of a physics-based model in $\mathcal{L}_\textit{phys}$ to incorporate prior knowledge in PiNNs~\cite{chen2023physics}. However, this requires access to a physics-based model, introducing engineering overhead. Moreover, it can enforce unwanted biases since the physics-based model might be inaccurate and steer PiNN predictions away from the truth. Consequently, in this work, we penalize the gradients of the temperature predictions instead, according to our definition of physical consistency in Section~\ref{sec:phys}, which bypasses the need for a physics-based model and only relies on measured quantities while still incorporating knowledge about the underlying laws of physics in PiNNs.}

\begin{remark}[Computational complexity] \label{rem:complexity}
To ensure a model is following the underlying physical laws at all times, one should check the gradients throughout the prediction horizon, and not only for the last predictions, as proposed in \eqref{equ:phys loss}. However, since each gradient computation requires one forward and one backward pass of the data, the computational complexity grows linearly with the number of predictions to analyze. Consequently, we only compute the gradients of the last predictions with respect to all the control inputs and ambient temperatures observed along the horizon to steer PiNNs, alleviating the associated computational burden. 
\end{remark}

    \section{Results}
    \label{sec:results}
    
This section provides a comprehensive analysis of the models presented in Table~\ref{tab:accuracy}. \new{All the results discussed hereafter were obtained by comparing the multi-step prediction performance of the different models on more than $750$ three-day-long time series from the validation set. Each model is recursively applied to predict the temperature in all the zones for $288$ steps, i.e., three days,\footnote{Given a warm start of \SI{3}{\hour} for the models based on LSTMs.} assuming knowledge of all the inputs, and compared to the true measured temperatures. The \textit{error} of a model is then defined as its average performance over the three zones.}

\begin{table}[]
    \centering
    \begin{tabular}{l|l|c|c|c} \hline
                        &                           & \textbf{Phys.} & &\\
    \textbf{Category}   & \textbf{Model}            & \textbf{cons.} & \textbf{MAE} & \textbf{MAPE}\\ \hline
                & \textit{Linear}                   & \textcolor{OliveGreen}{\cmark} & $1.79$ & $7.5\%$ \\
    Gray-box    & \textit{Res}                      & \textcolor{red}{\xmark} & $1.79$ & $7.7\%$ \\
                & \textit{Res-cons}                 & \textcolor{OliveGreen}{\cmark} & $1.50$ & $6.4\%$ \\ \hline 
                & \textit{ARX}                      & \textcolor{red}{\xmark} & $1.68$ & $7.1\%$ \\
    Black-box   & \textit{ARX-KF}                   & \textcolor{red}{\xmark} & $1.35$ & $5.6\%$ \\
                & \textit{LSTM}                     & \textcolor{red}{\xmark} & $1.27$ & $5.5\%$ \\
                & \textit{PiNN}                     & \textcolor{red}{\xmark} & $1.37$ & $5.8\%$ \\ \hline 
                & \textbf{\textit{X-PCNN} (Ours)}   & \textcolor{OliveGreen}{\cmark} & $\bm{1.17}$ & $\bm{4.9\%}$ \\
    PCNNs       & \textit{M-PCNN} (Ours)            & \textcolor{OliveGreen}{\cmark} & $1.25$ & $5.3\%$ \\
                & \textit{S-PCNN} (Ours)            & \textcolor{OliveGreen}{\cmark} & $1.22$ & $5.1\%$ \\ \hline 
    \end{tabular}
    \caption{Physical consistency, MAE, and MAPE of the methods investigated in this work.} 
    \label{tab:accuracy}
\end{table}

\del{In }Section~\ref{sec:accuracy} \new{starts with a discussion on}\del{, we start by discussing} the performance of the different models \del{ on the validation data.} In Section~\ref{sec:visualization}, we then provide a visual and qualitative discussion of the NN-based model \new{ predictions}\footnote{The LSTM, PiNN, and PCNN architectures. Despite also being composed of an NN, residual models are not considered as \textit{NN-based} models in this work since their main dynamics are still captured by the underlying linear model, and not the NN.} \del{,} before numerically investigating their physical consistency more in-depth in Section~\ref{sec:gradients}.
Finally, Section~\ref{sec:complexity} concludes with a brief overview of the computational complexity associated with all the models. Altogether, this will allow us to understand the trade-offs between the physical consistency, accuracy, and computational complexity of the various data-driven building modeling methods examined in this work. 

Note that the NN-based models were run with several random seeds, and, unless stated otherwise, the results discussed throughout this Section were obtained using the best-performing seed in each case. Remarkably, however, this does not impact our conclusions significantly since all the architectures proved to be robust to the choice of random seed, with standard deviations in the range of $0.01$ -- $0.04$ and $0$ -- $0.2\%$ for the MAE and MAPE, respectively, as detailed in Appendix~\ref{app:seed}. 


    \subsection{Performance analysis}
    \label{sec:accuracy}
    

The best performance of all the analyzed methods in terms of MAE and MAPE \del{on more than $750$ three-day-long time series from the validation set} is reported in Table~\ref{tab:accuracy}. 
As can be observed, all the proposed PCNN architectures attain state-of-the-art accuracy, both in terms of MAE and MAPE. They are followed by physically inconsistent black-box methods, especially the ones relying on very expressive NNs. 
As expected, the least expressive class of methods, gray-box models, performs the worst. 
Combining these results with the physical consistency of each method, we can conclude that the proposed PCNN architectures take the best out of both  gray- and black-box methods, attaining state-of-the-art performance while respecting the underlying physical laws by construction without trade-off, making them ideal thermal building models.

\new{While X-PCNNs achieve the best performance in Table~\ref{tab:accuracy}, we suspect these results to be influenced by the analyzed case study. Indeed, the temperature dynamics in UMAR are strongly impacted by solar gains, which reduces the importance of energy exchanges between the zones. This might explain why it is possible to fit the overall building dynamics well even when independently training one model for each zone, as for X-PCNNs, and why the \textit{post hoc} correction~\eqref{equ:c correction} only has a little impact on the final model performance. We suspect this required correction might have a stronger influence on multi-zone buildings where temperature dynamics are less impacted by weather conditions and more governed by energy exchanges between the zones, which might, in turn, decrease the quality of X-PCNNs.}

Remarkably, enforcing the physical consistency of LSTMs, as in PCNNs, seems to improve their accuracy in this case study despite the introduced constraints. This confirms the ongoing trend in ML research to include prior knowledge in NN architectures. Even if it might intuitively seem that introducing structural constraints should hinder the expressiveness of LSTMs, these results suggest that it can on the contrary be helpful. Moreover, one can draw similar conclusions with the two residual models investigated in this work, with \textit{Res-cons} clearly outperforming its physically inconsistent counterpart despite both models relying on the same linear basis. 
While these results are not reported here, ensuring the black-box modules to be independent of $D^z$, as mentioned in Remark~\ref{rem:D}, also increased the performance of X-PCNNs. Altogether, these results point towards performance benefits of grounding NN architectures in the underlying physics, ensuring that they learn meaningful solutions. 

As shown in Table~\ref{tab:accuracy}, the residual models (\textit{Res} and \textit{Res-cons}), which are conceptually close to PCNNs\new{,}\footnote{Especially \textit{Res-cons}, where the only difference in architecture comes from the solar irradiation processing.}\del{,} are unable to attain similar performance to PCNNs. This hints towards the benefits of learning all the parameters together in an end-to-end fashion instead of first identifying the linear part and then fitting the residual errors. 

PCNNs are on average $30$ -- $35\%$ and $17$ -- $22\%$ more accurate than the other physically consistent methods, namely the \textit{Linear} and \textit{Res-cons} model, respectively. To visualize the error propagation of these methods over three days, their MAE is plotted in Figure~\ref{fig:error propagation}. 
This shows that the proposed \text{PCNNs} not only perform better on average, but along the entire prediction horizon, except during the first few hours, where the linear and residual model attain similar performance. The main reason behind this behavior is the warm start of PCNNs, which often gives erroneous first predictions, but they quickly make up for it and show much stronger performance in the long run. At the end of the horizon, PCNNs indeed show an error $34$ -- $41\%$ and \text{$10$ -- $18\%$} lower than \textit{Linear} and \textit{Res-cons}, respectively, with the best performance again achieved by the X-PCNN.

The results of our case study suggest that the black-box modules of PCNNs are able to process the raw solar irradiation data 
and infer its impact on the zone temperatures. Indeed, they outperform gray-box models, which have access to engineered solar gains (Appendix~\ref{app:solar preprocessing}). 
Interestingly, since the impact of the sun is implicitly computed by their black-box modules, PCNNs could easily be applied to any building, even when shading comes into play, making the engineered preprocessing of solar data required for gray-box models much more complex. Furthermore, while the only nonlinear gains considered in this paper come from solar irradiation, the flexibility of the black-box module would also allow it to learn other gains, such as the ones stemming from occupants.


    \begin{figure}
    \begin{center}
    \includegraphics[width=\columnwidth]{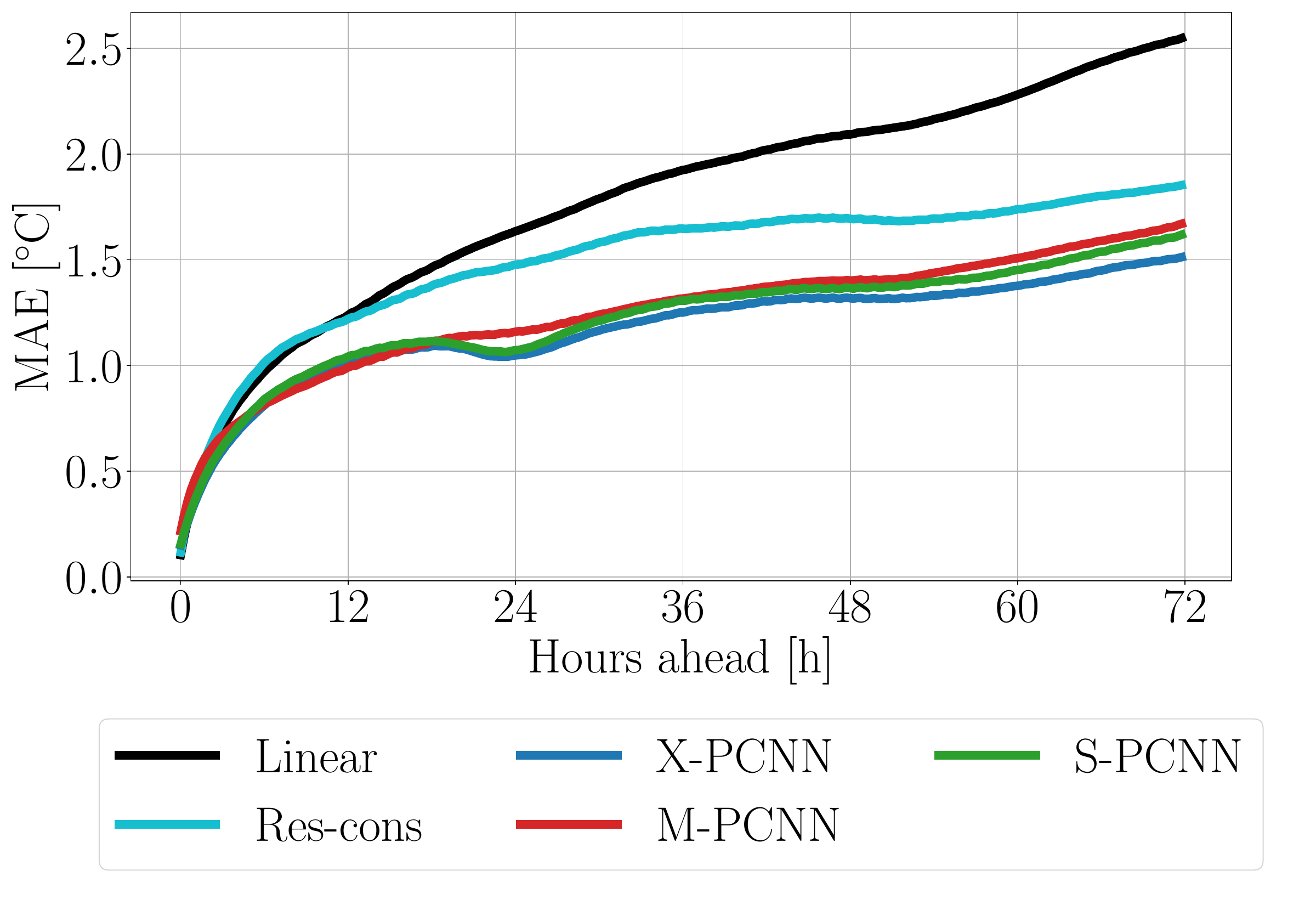}
    \caption{MAE of all the physically consistent methods \new{over the prediction horizon} averaged over the three zones and the time series of three days in the validation data set.}
    \label{fig:error propagation}
    \end{center}
    \end{figure}

    \subsection{The necessity of physical consistency}
    \label{sec:visualization}
    
Now that Section~\ref{sec:accuracy} established that PCNNs attain state-of-the-art performance in terms of accuracy, even outperforming pure black-box methods, let us visualize the behavior of one S-PCNN, one PiNN, and one LSTM for a given random seed in Figure~\ref{fig:propagation}. 
To that end, the thermal power is turned off in Zone $1$ and $2$ and we examine the impact of heating (red), cooling (blue), or providing no power input (black) in Zone $3$. Note that the heating pattern corresponds to the true power inputs measured in Zone $3$ in March $2021$, which we mirror to create the cooling pattern. 

\begin{figure*}
    \begin{center}
    \includegraphics[width=\textwidth]{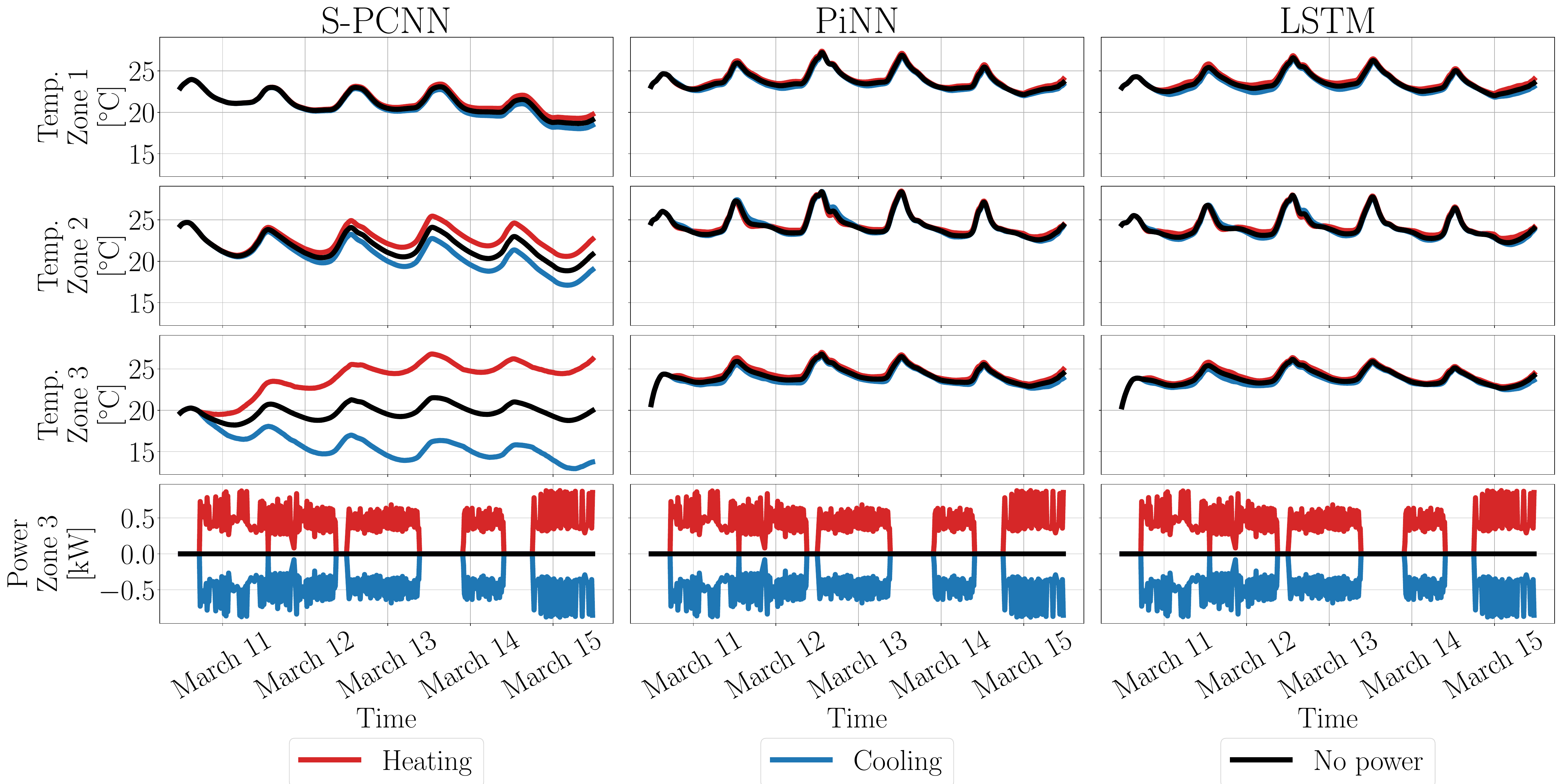}
    \caption{Visualization of heat propagation for the proposed S-PCNN architecture on the left compared to a PiNN in the middle and an LSTM on the right. The bottom plots show the heating (red) and cooling (blue) patterns applied to Zone~$3$ while the power is turned off in Zone~$1$ and $2$, compared to the situation when no power is applied (black). The other plots depict the corresponding temperature predictions of each model in each of the three zones. 
    }
    \label{fig:propagation}
    \end{center}
    \end{figure*}

As one can immediately realize, following the laws of thermodynamics, heating or cooling Zone~$3$ increases, respectively decreases its temperature in the S-PCNN model. This effect is then propagated to the adjacent Zone~$2$, and later to Zone~$1$, impacting their temperatures even though they are neither heated nor cooled, illustrating the effect of Corollary~\ref{cor:consistency} ensuring physically consistent predictions. 
Note that while only the temperature predictions of one S-PCNN are pictured in Figure~\ref{fig:propagation}, similar effects were observed for \text{X-PCNNs} and M-PCNNs. This is expected since all of them share the same properties, i.e.\new{,} the same physics-inspired module, to ensure they follow the criteria \eqref{equ:input consistency}-\eqref{equ:neighbors consistency}. 
    
On the other hand, all power inputs lead to almost indistinguishable temperature predictions for the PiNN and LSTM. Despite achieving a very good fit of the data (Section~\ref{sec:accuracy}), these models are hence obviously flawed and can be misleading in practical applications. We can sometimes even observe lower temperature predictions when heating is turned on than when the zones are cooled,
a clear sign of physical inconsistency\del{, e.g. in Zone $2$ at the end of March $11$ and $12$} \new{(see Appendix~\ref{app:visualization} for zoomed in results)}. 
This clearly exemplifies the issue of shortcut learning in the case of thermal modeling, where NNs manage to fit the data well without understanding the underlying physics. 

Furthermore, this illustrates how PiNNs only \textit{steer} the learning towards interesting solutions without providing any guarantees concerning the actual behavior of the model. In fact, trained PiNNs always gave very similar predictions to LSTMs in our experiments, as can also be seen in Figure~\ref{fig:propagation}, hinting that modifying the loss function $\new{\mathcal{L}_{\textit{PiNN}}}\del{\mathcal{L}_{PiNN}}$ of the model did not have much impact on the final solution found despite the large $\lambda$ used. While tuning this hyperparameter might lead to better results, it is a notoriously cumbersome task and would still never guarantee the physical consistency of the final model~\cite{faroughi2022physics}. Thus, it was not considered in this work.

Very importantly, these results point out a somewhat counter-intuitive and often overlooked characteristic of NNs: contrary to physics-based models, a good fit to the data does not necessarily imply that the quality of the model is good. 
In our case, the PiNNs and LSTMs were indeed able to fit the data well without considering the impact of heating and cooling, i.e.\new{,} solely mapping external conditions to building temperatures. One hence has to be careful when NNs are used to model physical systems and make sure the trained models do not simply find shortcuts to fit the data well without respecting the underlying physical laws. This calls for physically grounded architectures, such as PCNNs, for applications where the physical consistency of the model is critical.

We suspect that LSTMs and PiNNs were able to fit the data very well without considering the impact of heating and cooling because of the specific data used in this case study. First\del{ly}, windows cover the entire East facade of UMAR, rendering the building especially sensitive to solar gains and external weather conditions. Second\del{ly}, while different controllers have been applied during the collection period of the data set, all of them had the same objective of maintaining the building temperature in a comfortable range and hence reacted similarly to external conditions. Coupling these facts, it seems indeed plausible to accurately predict building temperatures solely based on external conditions and without considering heating and cooling inputs. \new{In other words, we suspect the very expressive LSTMs and PiNNs to have learned the \textit{closed-loop} response of the system instead of the expected \textit{open-loop} one, hence implicitly accounting for the influence of power inputs instead of explicitly modeling their effect.} This might explain \new{how they}\del{why the very expressive LSTMs and PiNNs} found non-physical shortcuts modeling the evolution of inside temperatures well. \new{Interestingly, the identified linear model also failed to capture any significant impact of heating and cooling, showing that it is also possible to fit this data well without accounting for these inputs but still following the underlying laws of physics (Appendix~\ref{app:visualization}).}

Interestingly, \textit{Res} did capture a much more significant impact of heating and cooling, but remains completely oblivious to the underlying physics, with cooling often resulting in higher temperatures than heating. This illustrates the need to also consider physical consistency when designing residual models, such as in the proposed \textit{Res-cons} architecture. In general, all these results hence suggest that physical consistency should always be considered when dealing with NNs for physical systems.

Note that the identification of the ARX-KF model assigned a negative scaling parameter for the power input to Zone~$2$, for example, meaning that heating this zone will lead to lower temperatures. We could observe similar issues with the parameters of the classical ARX model, indeed confirming that ARX models might not be physically consistent in practice, as claimed in Section~\ref{sec:benchmark}.

    \subsection{Numerical analysis of physical consistency}
    \label{sec:gradients}

This section investigates the gradients of the predictions of NN-based models numerically, to strengthen the theoretical and visual claims in Table~\ref{tab:accuracy} and Figure~\ref{fig:propagation}.
Since gradients can be retrieved automatically through the \texttt{torch.autograd} module~\cite{NEURIPS2019_9015}, it allows us to numerically assess if the models respect criteria \eqref{equ:input consistency} and \eqref{equ:outside consistency}, a necessary condition to ensure physical consistency. 
Following Remark~\ref{rem:complexity}, we investigate the gradients of the temperature predictions at the end of the three-day-long horizon 
with respect to the power inputs and ambient temperatures observed at each time step. Note that this corresponds to the gradients used to steer the learning of PiNNs in 
\eqref{equ:gradients}, except for the X-PCNN, for which fewer gradients can be computed, as detailed in Appendix~\ref{app:x-pcnn gradients}. \new{However, their magnitude does not have any physical meaning since NN-based models work with normalized data.}

    \begin{figure}
    \begin{center}
    \includegraphics[width=\columnwidth]{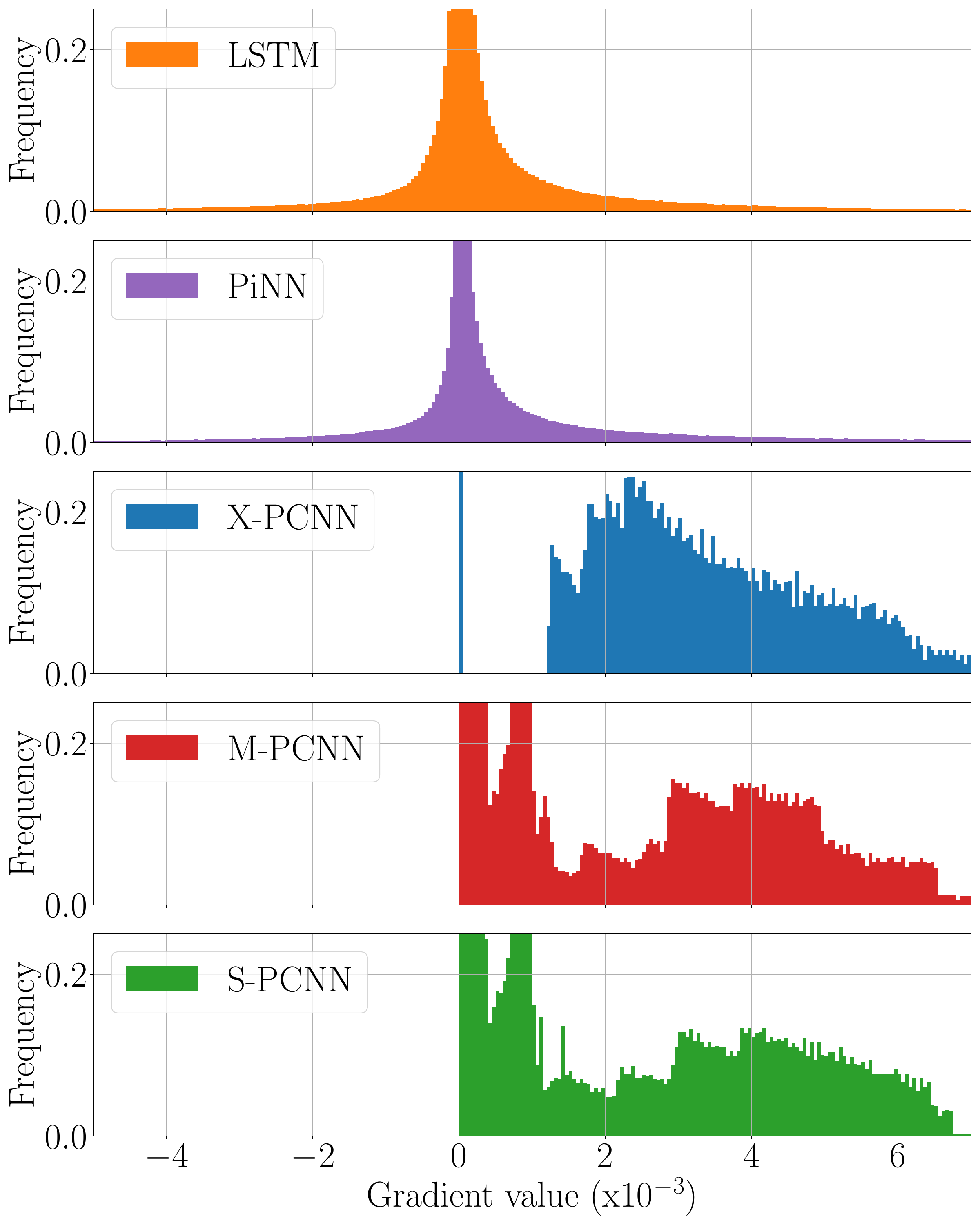}
    \caption{Distribution of the gradients of the temperatures at the end of the prediction horizon with respect to power inputs and external temperatures observed along the horizon for the NN-based models.}
    \label{fig:gradients}
    \end{center}
    \end{figure}
    
Overall, 
this gives us access to more than two million gradient values for each model, except the X-PCNN, with slightly over one million values\new{, as computed in Appendix~\ref{app:gradients}}. The resulting density histograms are shown in Figure~\ref{fig:gradients}, where one can directly observe negative gradients only for the two black-box models not grounded in the underlying physics. In fact, penalizing negative gradients in $\new{\mathcal{L}_{\textit{PiNN}}}\del{\mathcal{L}_{PiNN}}$ decreased the magnitude of the PiNN gradients, steering them to zero, but did not change the proportion of negative ones. In other words, it did not improve the physical consistency of PiNNs since they still violate conditions \eqref{equ:input consistency} and \eqref{equ:outside consistency} as often as classical LSTMs. \new{Remarkably, the small magnitude of the PiNN and LSTM gradients corroborate what can be seen in Figure~\ref{fig:propagation}, with very little impact of heating and cooling for these models.} On the other hand, \new{thanks to their physics-inspired module,} the proposed PCNN architectures keep all the gradients that require positivity in $\mathbb{R}_+$ \new{and with larger magnitudes}, as desired \new{and observed in Figure~\ref{fig:propagation} for the S-PCNN}, providing a numerical argument supporting their physical consistency. 

    \subsection{Computational complexity}
    \label{sec:complexity}
    
As final comparison metric between the models, Figure~\ref{fig:time} presents the time required by each model per training iteration. Importantly, these numbers are subject to implementation considerations and hence have to be taken with a grain of salt since we did not optimize the models. Nonetheless, all of them used the same backbone architecture, which allows relative comparisons, \new{for example}\del{e.g.} between the three proposed PCNNs, between the two residuals models, or between LSTMs and PiNNs. Note that the linear and ARX models are not considered here since their ``training'' procedure is very different: it does not require access to a GPU and does not rely on gradient descent

First\del{ly}, as expected, PiNNs take more time to run than classical LSTMs since each batch has to be forwarded and backwarded through the networks twice, once to compute the predictions used in $\new{\mathcal{L}_{\textit{data}}}\del{\mathcal{L}_{data}}$ 
and another time to calculate the gradients in $\new{\mathcal{L}_{\textit{phys}}}\del{\mathcal{L}_{phys}}$. Second\del{ly}, residual models need access to the predictions of the underlying linear model at each step to compute the residual errors before fitting them, which also entails a clear computational overhead compared to LSTMs. 
Finally, the PCNN architectures all require to compute both the black-box module output $D_k$ and the physics-inspired module predictions $E_k$ at each step $k$ along the horizon, which also entails additional overhead on top of classical black-box models. Interestingly, this is comparable to what happens in residual models, explaining to some extent why the latter and S-PCNNs require similar amounts of resources. 

Compared to S-PCNNs, M-PCNNs and \text{X-PCNNs} are significantly more computationally intensive.
This intuitively follows from the shared black-box module of \text{S-PCNNs} reducing the number of parameters to fit compared to \text{M-PCNNs}. On the other hand, \text{X-PCNNs} require learning several models separately instead of everything together, which duplicates the computational overhead needed to create and move data to the GPU at each iteration and leads to an increased computational burden compared to \text{M-PCNNs}. Stemming from these remarks, we would expect these differences to grow if we were to apply PCNNs to larger buildings with more thermal zones.
    
\begin{figure}
    \begin{center}
    \includegraphics[width=\columnwidth]{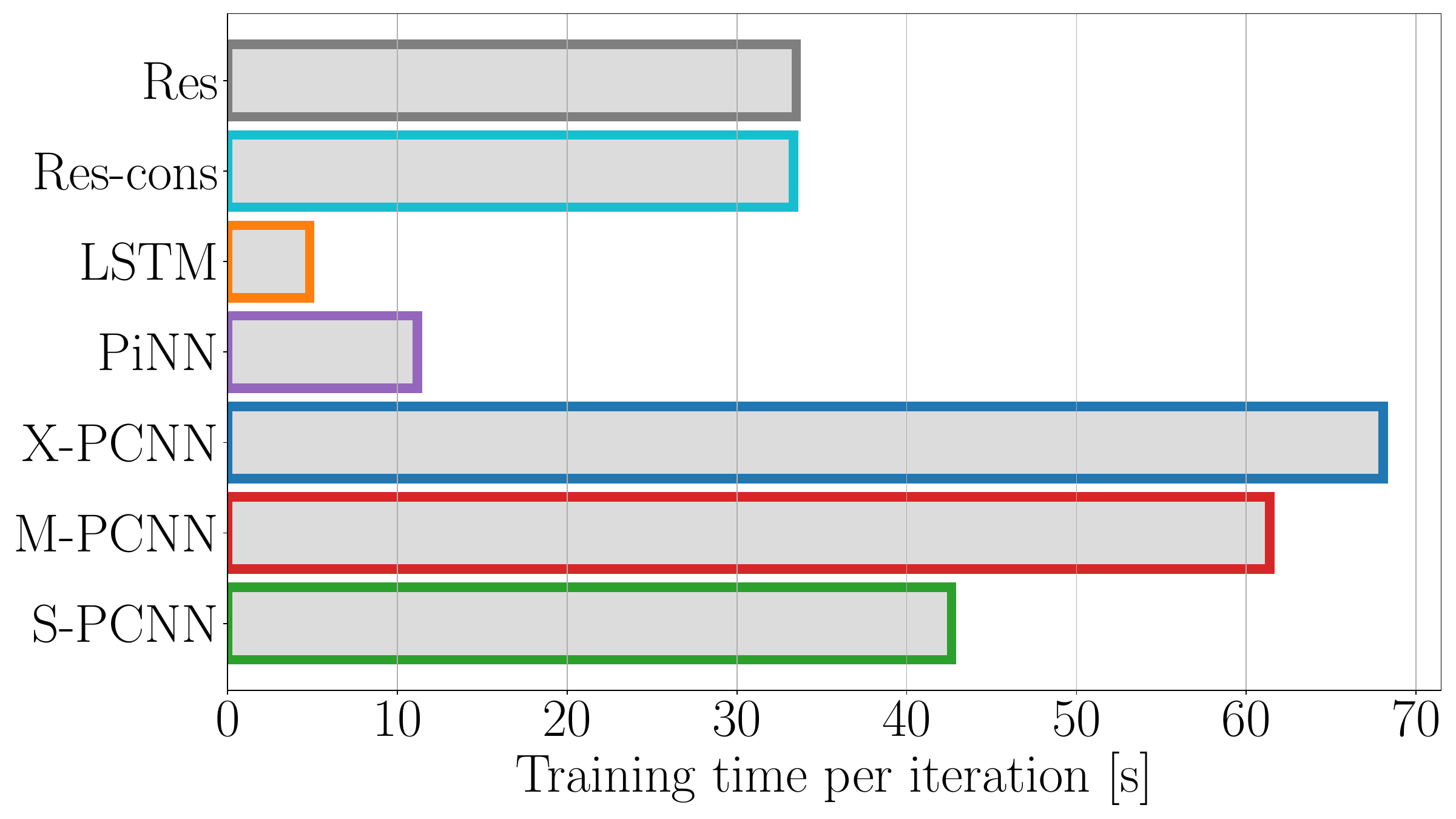}
    \caption{Training time per iteration of the methods relying on a GPU.}
    \label{fig:time}
    \end{center}
    \end{figure}



\begin{remark}[Parallelizing X-PCNNs]
    The training times reported here correspond to the total time required to train each model for one iteration, i.e.\new{,} the sum of training times of single-zone PCNNs in the case of X-PCNNs, to represent the total amount of computations needed. In practice, however, the different single-zone PCNNs can easily be trained in parallel since they are independent, which can significantly decrease the effective training time of X-PCNNs (dividing it approximately by three in our setting with three zones). This would make them the fastest multi-zone PCNNs to deploy but at the cost of additional computational complexity.
\end{remark}

\section{Conclusion}
    \label{sec:conclusion}
    
This work presented extensions of single-zone PCNNs to the multi-zone setting, thereby providing fully data-driven control-oriented multi-zone building thermal models. The main idea of PCNNs is to let a physics-inspired and a black-box module run in parallel, the former guaranteeing the compliance of the output with the underlying physical laws \new{---}\del{--} the laws of thermodynamics in the case of building temperature modeling \new{---}\del{--} and the latter capturing unknown nonlinear dynamics, typically relying on NNs.

The proposed multi-zone PCNNs 
respect the underlying physics by design and at all times despite requiring little engineering, contrary to classical physically consistent methods. On the other hand, they outperformed state-of-the-art black-box methods in terms of accuracy on a case study, hinting that the constrained structure introduced to ensure they follow some ground rules does not hinder their expressiveness. Our analyses showed little difference between S-, M-, and X-PCNNs in general, with \text{S-PCNNs} entailing the least computational complexity and X-PCNNs attaining the best accuracy \new{on the analyzed case study}. Remarkably, all of them showed significantly better performance than classical physically consistent data-driven methods, with accuracy improvements of $30$ -- $35\%$ and $17$ -- $22\%$ compared to a linear and a residual model, respectively. While these results were obtained on a specific building, the performance gap suggests that this trend would be observed for other applications. PCNNs should thus remain the best modeling choice in general, even if classical black-box methods might attain a better accuracy on different data sets. 

Our investigations also illustrated a well-known pitfall of classical PiNNs and LSTMs, which can find shortcuts to fit the data well without respecting the underlying physical laws. This exemplifies the need to not solely consider the fit to the data as a measure of the quality of NNs but also ensure that their predictions make sense from a physical point of view. Our findings hence support the current trend to incorporate inductive biases, i.e.\new{,} prior knowledge, in NNs to alleviate their infamous 
generalization issues, leading to more principled architectures like the proposed PCNNs.


In light of these results, PCNNs pave the way for NN-based methods that can simultaneously provide state-of-the-art performance \textit{and} physical guarantees. Furthermore, while only solar irradiation measurements and time information were fed to the black-box module of \text{PCNNs} throughout this study, showcasing the ability of the proposed approach to handle highly nonlinear effects, other inputs could be integrated straightforwardly. 

Thanks to their flexibility, we hence believe PCNNs to be an essential step towards the safe deployment of NNs in real-world applications, closing the sim2real gap of advanced control algorithms, and hope to spark an interest both in the building modeling community and beyond. It would indeed be interesting to investigate the capabilities of PCNNs to model different buildings, incorporate additional nonlinearities in their black-box modules or rules in the physics-inspired ones \new{--- for example leveraging Irreversible port-Hamiltonian dynamics \cite{zakwan2022physically}---}, and tackle other complex physical systems.

    \section*{Acknowledgements}
    
This research was supported by the Swiss National Science Foundation under NCCR Automation, grant agreement 51NF40\_180545, and in part by the Swiss Data Science Center, grant agreement C20-13.

    \section*{Declaration of competing interests}
    
The authors declare that they have no known competing financial interests or personal relationships that could have appeared to influence the work reported in this paper.

\printcredits

\bibliographystyle{els-cas-templates/model1-num-names}

\bibliography{biblio.bib}

\appendix
    \section*{Appendices}
    
    \section{Proofs of the main theoretical results}
    \label{app:proofs}
    
    \subsection{Proof of Proposition~\ref{prop:heat propagation}}
    \label{app:proof heat propagation}
    
The proof works by induction on $i$. 
Based on \eqref{equ:common-T}, we can immediately write, $\forall z,y\in\mathcal{B}$:
\begin{align}
    \frac{\partial\bm{T}_{k+j+1}^z}{\partial\bm{T}_{k+j}^y} &= 
    \begin{cases}
    1-b^z-\sum_{y\in\mathcal{N}(z)}c^{zy}, &\text{if }y=z, \\
    c^{zy}, &\text{if }y\in\mathcal{N}(z), \label{equ:base} \\
    0, &\text{otherwise,}
    \end{cases} 
\end{align}
where we used the definition of $\bm{\Delta T}$ in \eqref{equ:deltaT}. By definition, if \eqref{equ:prop1-cond1} and \eqref{equ:prop1-cond2} hold, we hence get positive derivatives if $y=z$ or $y\in\mathcal{N}(z)$ and zeros for any other choice of $y$, satisfying \eqref{equ:prop1} and completing the base case of the induction. 

Let us now assume that: 
\begin{align}
    \frac{\partial\bm{T}_{k+h}^x}{\partial\bm{T}_{k+j}^y} \geq 0,\ \forall y,x\in\mathcal{B},\ \forall j < h < i, \label{equ:hypo}
\end{align}
with equality if and only if $y\notin\mathcal{N}^{(h-j)}(x)$, and show that the proposition holds for time step $i$. Since we know the temperature in zone $z$ at time $k+i$ is potentially impacted by the temperature in the entire building at the previous step, we can decompose the partial derivative of interest as follows:
\begin{align}
    \frac{\partial\bm{T}_{k+i}^z}{\partial\bm{T}_{k+j}^y} &= \sum_{x\in\mathcal{B}} \frac{\partial\bm{T}_{k+i}^z}{\partial\bm{T}_{k+i-1}^x} \frac{\partial\bm{T}_{k+i-1}^x}{\partial\bm{T}_{k+j}^y}, \label{equ:sum}
\end{align}
for all $y,z\in\mathcal{B}$. Since \eqref{equ:common-T} is time-invariant, we know that:
\begin{align*}
    \frac{\partial\bm{T}_{k+i}^z}{\partial\bm{T}_{k+i-1}^x} &= \frac{\partial\bm{T}_{k+j+1}^z}{\partial\bm{T}_{k+j}^x} \geq 0, 
\end{align*}
with equality if and only if $x\notin\mathcal{N}(z)$ by the base case of the induction \eqref{equ:base} if \eqref{equ:prop1-cond1} and \eqref{equ:prop1-cond2} hold. Similarly, by the induction hypothesis \eqref{equ:hypo}, we know that:
\begin{align*}
    \frac{\partial\bm{T}_{k+i-1}^x}{\partial\bm{T}_{k+j}^y} \geq 0,\ \forall y,x\in\mathcal{B},
\end{align*}
with equality if and only if $y\notin\mathcal{N}^{(i-j-1)}(x)$. Putting the last two equations together, we see that:
\begin{align*}
    \frac{\partial\bm{T}_{k+i}^z}{\partial\bm{T}_{k+j}^y} &\geq 0, 
\end{align*}
with equality only if each term of the sum in Equation~\eqref{equ:sum} is zero. By the previous arguments, this means $y\notin\mathcal{N}^{(i-j-1)}(x)$ or $x\notin\mathcal{N}(z)$ for all zones $x$. This is equivalent to say that there is no path from $y$ to $z$ in $(i-j)$ steps, i.e.\new{,} $y\notin\mathcal{N}^{(i-j)}(z)$, which concludes the inductive step.

    \subsection{Proof of Proposition~\ref{prop:inputs consistency}}
    \label{app:proof inputs consistency}

We start by noticing that $\forall y\in\mathcal{B}$, \eqref{equ:common-T} implies:
\begin{align}
    \frac{\partial\bm{T}_{k+j+1}^y}{\partial\bm{u}_{k+j}^y} &= 
    \begin{cases}
    a^y_h,          &\text{if } \bm{u}_{k+j}^y > 0, \\
    a^y_c,          &\text{if } \bm{u}_{k+j}^y < 0, \\
    0, &\text{otherwise},  
    \end{cases} \label{equ:proof1}\\
    \frac{\partial\bm{T}_{k+j+1}^x}{\partial\bm{u}_{k+j}^y} &= 0 \ \qquad \forall x\in\mathcal{B}, x\neq y, \label{equ:proof2}\\
    \frac{\partial\bm{T}_{k+j+1}^y}{\partial T^{out}_{k+j}} &= b^y ,
\end{align}
Note that this proves that \eqref{equ:prop2-1} and~\eqref{equ:prop2-2} for the case $i=j+1$ if $a^y_h, a^y_c, b^y > 0$, $\forall y\in\mathcal{B}$. When $i>j+1$, Proposition~\ref{prop:heat propagation} implies:
\begin{align}
    \frac{\partial\bm{T}_{k+i}^z}{\partial\bm{T}_{k+j+1}^y} \geq 0, \qquad\forall z,y\in\mathcal{B},\ \forall 0\leq j<i-1,
\end{align}
with equality if and only if $y\notin\mathcal{N}^{(i-j-1)}(z)$ if the conditions in \eqref{equ:prop1-cond1} and \eqref{equ:prop1-cond2} hold.

Relying on the fact that the temperatures at time $k+i$ are potentially influenced by the temperatures in the whole building at time $k+j+1$, we have:
\begin{align}
    \frac{\partial\bm{T}_{k+i}^z}{\partial\bm{u}_{k+j}^y} &= \sum_{x\in\mathcal{B}}\frac{\partial\bm{T}_{k+i}^z}{\partial\bm{T}_{k+j+1}^x} \frac{\partial\bm{T}_{k+j+1}^x}{\partial\bm{u}_{k+j}^y}, \\
    &= \frac{\partial\bm{T}_{k+i}^z}{\partial\bm{T}_{k+j+1}^y} \frac{\partial\bm{T}_{k+j+1}^y}{\partial\bm{u}_{k+j}^y} \geq 0,
\end{align} 
where the second equality follows from \eqref{equ:proof2} and the inequality holds as long as \eqref{equ:prop1-cond1} and \eqref{equ:prop1-cond2} are respected and $a^y_h, a^y_c > 0$, $\forall y \in\mathcal{B}$. Furthermore, by Proposition~\ref{prop:heat propagation}, equality is only reached if $y\notin\mathcal{N}^{(i-j-1)}(z)$.

Similarly, we have:
\begin{align}
    \frac{\partial\bm{T}_{k+i}^z}{\partial T^{out}_{k+j}} &= \sum_{y\in\mathcal{B}}\frac{\partial\bm{T}_{k+i}^z}{\partial\bm{T}_{k+j+1}^y} \frac{\partial\bm{T}_{k+j+1}^y}{\partial T^{out}_{k+j}}, \\
    &= \sum_{y\in\mathcal{B}}\frac{\partial\bm{T}_{k+i}^z}{\partial\bm{T}_{k+j+1}^y} b^y > 0,
\end{align} 
where the strict inequality is respected as long as $b^y>0$, $\forall y \in\mathcal{B}$. Indeed, since $z\in\mathcal{N}(z)$ by definition, Proposition~\ref{prop:heat propagation} then implies that at least one of the terms in the sum is strictly positive, while the others are nonnegative.

    \section{Details on the data processing}
    \label{app:processing}

Once the data had been subsampled and processed as in~\cite[App. C]{di2021physically} and discarding the $23$\% of incomplete measurements, i.e. where at least the information from one sensor is missing, we were left with over \new{\numprint{80000}}\del{$80'000$} data points. Since the proposed PCNN architectures are based on NNs in our implementations (see Section~\ref{sec:implementations}), they are not able to handle missing values, which prompted us to create a data set of time series without missing values. 

As we aimed to design models that are able to predict the temperature dynamics over three day-long horizons, we truncated each sequence to a maximum of three days, and separated the heating and cooling seasons. We allowed the time series to overlap each hour, i.e. each four steps, to increase the data efficiency of the approach. Finally, since we implemented a warm-start period of \SI{3}{\hour} for all the models, 
we also made sure the last \SI{3}{\hour} of data exist for each time series. To avoid very short time series, we also ensured they always span at least \SI{12}{\hour}. Altogether, this allowed us to create more than \new{\numprint{11000}}\del{$11'000$} sequences of data without missing values, which were split in a training and a validation set with proportions $80\%$-$20\%$, respectively denoted $\mathcal{D}_{t}$ and $\mathcal{D}_{v}$, and where $\mathcal{D}_{t}\cap\mathcal{D}_{v}=\emptyset$. \new{For all NNs, the validation set is used to select the best set of weights along the training procedure.}

    \section{Solar irradiation preprocessing}
    \label{app:solar preprocessing}
    
To compute the solar irradiation on the windows of a thermal zone $z$ from the measured irradiation on a horizontal surface $Q^{sun}$, we rely on the altitude and azimuth angles, respectively $\phi$ and $\theta$, of the sun. The former captures the elevation of the sun above the horizon while the latter represents its deviation from the north, in the clockwise direction.

    \begin{figure}
    \begin{center}
    \includegraphics[width=0.7\linewidth]{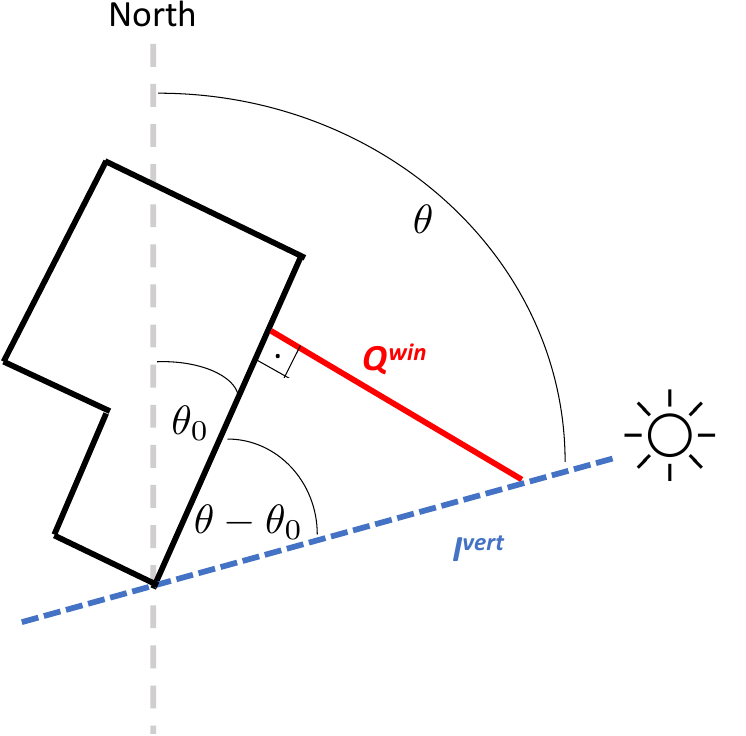}
    \caption{Sketch of the azimuth angles used to compute the solar irradiation on the windows of a building from the irradiation on a fixed vertical surface.}
    \label{fig:Sun}
    \end{center}
    \end{figure}
    
First\del{ly}, using the altitude of the sun and basic trigonometry, one can easily show that the measured irradiation on a horizontal surface corresponds to $Q^{sun} = I\sin\phi$, where $I$ is the global solar irradiation. Similarly, we know that the irradiation on a vertical surface following the sun, i.e.\new{,} tracking its azimuth angle to stay perpendicular to the incoming rays, can be computed as $I^{vert} = I\cos\phi$. We can hence write the solar irradiation on a vertical surface following the sun as follows:
\begin{align}
    I^{vert} &= Q^{sun}\frac{\cos\phi}{\sin\phi}.
\end{align}
Since building facades and windows have a fixed orientation in practice and do not follow the sun azimuth, we again use basic trigonometry to compute the irradiation on a north--south aligned surface facing east as $I^{vert}\sin\theta$. Finally, if the facade is not exactly facing east, we also need to account for its own "azimuth" $\theta_0$, i.e.\new{,}how much it is rotated clockwise starting from an east-facing position (Figure~\ref{fig:Sun}), which leads to:
\begin{align}
    Q^{win} &= I^{vert}\sin(\theta - \theta_0) \\
    &= Q^{sun}\frac{\cos\phi}{\sin\phi}\sin(\theta - \theta_0).
\end{align}
Once this has be done for each zone $z$, we can populate the required vector $\bm{Q}^{win}$ used by gray-box architectures in this work.

As one can readily observed, this processing only requires access to the elevation and azimuth angles of the sun, and to the orientation of the facade of interest. Furthermore, both solar angles solely depend on the geographical position of the building, i.e.\new{,} its latitude and longitude, and the time at which the measurement was taken. The position and orientation of a building can easily be found on plans or Google Maps, and we used the \texttt{Astral} Python library (\url{https://astral.readthedocs.io/en/latest/}) to compute the solar angles corresponding to each time step in our data.

Note that, while this processing works very well for unobstructed facades when its orientation is known, it cannot be used when for example other buildings or trees exist in front of the windows and create shading patterns. 
In that case, one has to rely on architectures which are able to automatically process horizontal solar irradiation measurements depending on time information, such as the LSTMs used in the black-box module of PCNNs. Nonetheless, we can use it in this paper since UMAR is not obstructed, leading to a very efficient computation of the true solar irradiation patterns on the windows of each zone.

    \section{Linear model identification}
    \label{app:linear}

As is classically done in linear system identification, we first\del{ly} used the least squares method to find the parameters $a^z_h$, $a^z_c$, $b^z$, $c^{zy}$, $e^z$ best fitting the training data for each thermal zone $z$ and neighboring zone $y\in\mathcal{N}(z)$, such as in~\cite[App. A.2]{di2021physically}. However, ensuring none of these parameters is negative, which is necessary to respect the underlying physics, produced $c^{23}=0$. This is clearly not physically meaningful, as it would mean there is no heat transfer from Zone $3$ to Zone $2$. Consequently, we also implemented a BO framework, relying on the \texttt{bayes\_opt} Python library~\cite{bayesopt}. This allowed us to extensively search for the best physically consistent parameters for each zone over a five-step prediction horizon, constraining all the parameters to be positive, for \new{\numprint{2300}}\del{$2300$} iterations starting with $200$ random initial points. 

    \section{Impact of the random seed}
    \label{app:seed}

The mean performance of the five NN-based model architectures, as well as the corresponding standard deviation, is presented in Table~\ref{tab:stds}. While the LSTM and S-PCNN models were run on five seeds due to their slightly higher sensitivity, the other results were obtained over three seeds leading to very consistent performance. As in the the original PCNN paper~\cite{di2021physically}, this hints at the robustness of the proposed approach, which does not seem significantly impacted by the random seed, or at least similarly to classical NN models. On this case study, the proposed X-PCNN and M-PCNN seem more robust to the choice of random seed than the S-PCNN. However, as shown in Table~\ref{tab:accuracy}, the latter sometimes outperforms M-PCNNs. Nonetheless, overall, \text{X-PCNNs} seem to have the upper hand, always attaining state-of-the-art performance even under different random seeds.

\begin{table}[]
    \centering
    \begin{tabular}{l|c|c} \hline
    \textbf{Model} & \textbf{MAE} & \textbf{MAPE} \\ \hline
    \textit{LSTM}                     & $1.33 \pm  0.04$ & $5.7\% \pm 0.2\%$ \\
    \textit{PiNN}                     & $1.38 \pm 0.01$ & $5.9\% \pm 0.1\%$ \\ \hline 
    \textbf{\textit{X-PCNN} (Ours)}   & $\bm{1.18 \pm 0.01}$ & $\bm{5.0\% \pm 0.0\%}$ \\
    \textit{M-PCNN} (Ours)            & $1.26 \pm 0.01$ & $5.4\% \pm 0.0\%$ \\
    \textit{S-PCNN} (Ours)            & $1.27 \pm 0.04$ & $5.4\% \pm 0.2\%$ \\ \hline 
    \end{tabular}
    \caption{MAE and MAPE of the methods investigated in this work, average over the three thermal zone, the three-day long horizon, and more than $750$ time series.}
    \label{tab:stds}
\end{table}

   \section{\new{Visualization of predictions}}
   \label{app:visualization}

    \begin{figure*}
    \begin{center}
    \includegraphics[width=\textwidth]{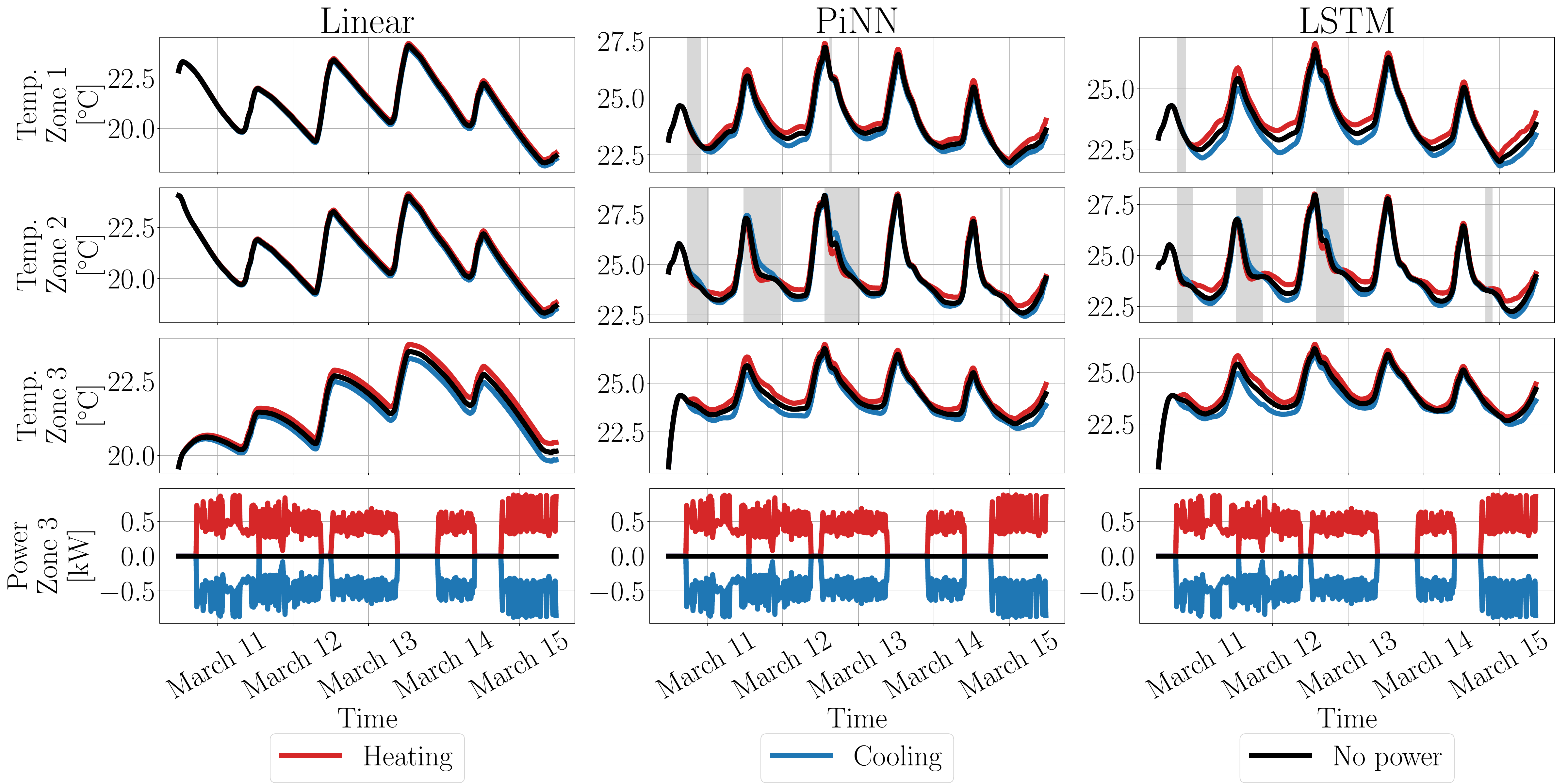}
    \caption{\new{Visualization of heat propagation for the linear on the left compared to a PiNN in the middle and an LSTM on the right. The bottom plots show the heating (red) and cooling (blue) patterns applied to Zone~$3$ while the power is turned off in Zone~$1$ and $2$, compared to the situation when no power is applied (black). The other plots depict the corresponding temperature predictions of each model in each of the three zones. Gray-shaded area mark physical inconsistencies of the PiNN ans LSTM.}}
    \label{fig:propagation2}
    \end{center}
    \end{figure*}

\new{To complement Figure~\ref{fig:propagation}, the same experiment was carried on with the linear model, and the corresponding predictions can be found in Figure~\ref{fig:propagation2} (left). Note that each subplot is using a custom scale to better visualize the impact of different power inputs. We additionally shaded physically inconsistent behaviors in each subplot in gray, i.e., whenever the predicted temperature when cooling is applied is higher than when heating is applied or no power input is used, or when the temperature when heating is applied is lower than when no power is used. This confirms that the identified linear model failed to fully capture the impact of heating and cooling but still behaves in a physically consistent manner, e.g., with heating leading to higher temperatures than cooling, similar to the behavior that can be observed for the S-PCNN in Figure~\ref{fig:propagation}. On the other hand, both the PiNN and LSTM show inconsistent behaviors, especially in Zone $2$ around the beginning of the prediction horizon.}

    \section{X-PCNN gradients}
    \label{app:x-pcnn gradients}

In the case of X-PCNNs, at inference time, we use each single-zone PCNN to predict the next temperature in the corresponding zone. The new temperatures in the building are then updated in the data of all the single-zone PCNNs 
so they can predict the next step. This is required because the single-zone PCNNs cannot evolve independently over the prediction horizon since they depend on temperatures in neighboring zones at teach step. 
However, overwriting the data at each step breaks the automatic backpropagation of Python, and we cannot automatically compute the gradient of the temperature in zone $z$ with respect to power inputs or temperatures in another zone $z'$ without implementation overhead. We can only retrieve gradients with respect to each single-zone PCNN's inputs, i.e.\new{,} the power $u^z$, and the ambient temperature. \new{Note that, intuitively, these available gradients are expected to be larger in magnitude than the gradients with respect to power inputs in other zones since they have a direct impact on the zone of interest. This explains the absence of low gradient values ($<10^{-3}$) in Figure~\ref{fig:gradients} for X-PCNNs compared to M- and S-PCNNs.} Even if we can only compute parts of the gradients automatically, we still show them in Figure~\ref{fig:gradients} for reference. Note that as we already know X-PCNNs are physically consistent since they satisfies the criteria of Corollary~\ref{cor:consistency}, these implementation considerations do not put the architecture in jeopardy.

    \section{\new{Number of numerical gradient values}}
    \label{app:gradients}

\new{The numerical investigation of NN-based model gradients in Section~\ref{sec:gradients} is carried out on the validation data set of more than $750$ three-day long sequences ($288$ steps). Following Remark~\ref{rem:complexity}, for each of the three zones, we compute the gradients of its last temperature predictions with respect to power inputs in all the zones ($3$~values) and the ambient temperature ($1$~value) at each step, giving rise to more than $750\times3\times288\times(3+1)=\numprint{2592000}$ values. In the case of X-PCNNs, we only have access to half of these values since we do not compute gradients with respect to power inputs in other zones (Appendix~\ref{app:x-pcnn gradients}), which still leaves us with more than $1$ million values.}

\end{document}